\documentclass[journal]{IEEEtran}

\usepackage{amsmath,amsfonts}
\usepackage{algorithm}
\usepackage{algorithmicx}
\usepackage{algpseudocode}
\usepackage{array}
\usepackage{makecell}
\usepackage{minibox}
\usepackage{multirow}
\usepackage{subcaption}
\usepackage{booktabs}
\usepackage{multirow}
\usepackage{setspace}
\usepackage{textcomp}
\usepackage{stfloats}
\usepackage{capt-of}
\usepackage{url}
\usepackage{verbatim}
\usepackage{graphicx}
\usepackage{cite}
\usepackage{marvosym}
\usepackage{xcolor}
\usepackage{tcolorbox}
\usepackage{colortbl}
\usepackage{booktabs}
\usepackage{pifont}
\usepackage{amsthm}
\usepackage{utfsym} 
\usepackage{bbm}
\usepackage[pagebackref=false,breaklinks=true,letterpaper=true,colorlinks,linkcolor=blue,citecolor=blue,bookmarks=false]{hyperref}
\usepackage[resetlabels]{multibib}
\newcites{supp}{References}

\definecolor{Draft}{RGB}{0, 123, 0}
\definecolor{LightCyan}{rgb}{0.88,1,0.88}
\definecolor{beaublue}{rgb}{0.9, 0.95, 0.9}
\definecolor{blackish}{rgb}{0.2, 0.2, 0.2}
\definecolor{grayish}{rgb}{0.95, 0.95, 0.95}

\newtheorem{theorem}{Theorem}

\newtheorem{proposition}{Proposition}

\title{Two Sides of the Same Coin: Co-Evolving Search for Cross-Task Attacks on Vision-Language Models}

\author{
Xuanhui Lin,
Junhao Dong\textsuperscript{\Letter},
Mingrong Gong,
Yucheng Chen,
Xinghua Qu,
and Yew-Soon~Ong,~\IEEEmembership{Fellow,~IEEE}
\thanks{Xuanhui Lin, Junhao Dong, Mingrong Gong, and Yew-Soon Ong are with the College of Computing and Data Science, Nanyang Technological University, Singapore. Yew-Soon Ong and Junhao Dong are also with the Center for Frontier AI Research, Agency for Science, Technology and Research (A*STAR), Singapore (e-mail: \{N2502791J, GONG0116\}@e.ntu.edu.sg; \{junhao003, asysong\}@ntu.edu.sg).}
\thanks{Yucheng Chen is with the School of Electrical Engineering and
Automation, Fuzhou University (e-mail: 832302208@fzu.edu.cn).
Xinghua Qu is with ByteDance, Singapore
(e-mail: \mbox{quxinghua17@gmail.com}).}
\thanks{\Letter $\;\leftarrow$ corresponding author.}
}

\begin{document}

\maketitle

\begin{abstract}
Vision-language models (VLMs) exhibit strong generalization across multimodal tasks but remain vulnerable to adversarial perturbations. Existing attacks typically follow single-trajectory gradient optimization or task-specific objectives, limiting search-space exploration and cross-task transferability. We propose an evolutionary-computation-guided cross-modal attack framework for unified VLMs. The framework adaptively searches both textual and visual spaces. On the textual side, it evolves hard negative semantic embeddings around the source-category representation to provide diverse cross-modal repulsion. On the visual side, it maintains a population of object-region perturbations and combines momentum-based gradient updates with evolutionary selection, mutation, and crossover to more reliably explore multiple feasible trajectories. Jointly optimizing semantic negative guidance and localized perturbations generates adversarial examples that consistently shift source-object semantics toward target categories across vision-language tasks. Theoretical analyses show that the co-evolutionary search preserves perturbation feasibility, prevents degradation of the best observed fitness, and increases the probability of reaching high-margin adversarial regions compared with single-trajectory optimization. Experiments on Florence-2, OFA, and UnifiedIO-2 demonstrate strong overall attack performance across image captioning, object detection, region categorization, and object localization. Ablation studies further verify the complementary effectiveness of text-side semantic evolution and image-side perturbation evolution, as well as the framework’s efficiency and cross-task transferability.
\end{abstract}


\section{Introduction}
Vision-Language Models (VLMs) have demonstrated remarkable capabilities across diverse vision-language tasks and application domains\cite{kuckreja2024geochat,cui2024survey,hou2025vision}. Despite their impressive progress, recent studies have revealed that VLMs are inherently susceptible to adversarial perturbations \cite{szegedy2013intriguing}, where imperceptible input modifications can mislead models into making highly confident incorrect predictions \cite{qi2024visual,kurakin2018adversarial}. This issue becomes even more pronounced in VLMs due to their multimodal inputs and increasingly complex architectures, which enlarge the potential attack surface. Such vulnerabilities pose substantial risks to the reliability and trustworthiness of VLMs in security-critical scenarios \cite{dong2025stabilizing,diaz2023connecting}.

\begin{figure}[!t]
\vspace{0.1cm}
\begin{subfigure}[t]{0.49\linewidth} 
    \centering
    \includegraphics[width=1\linewidth]{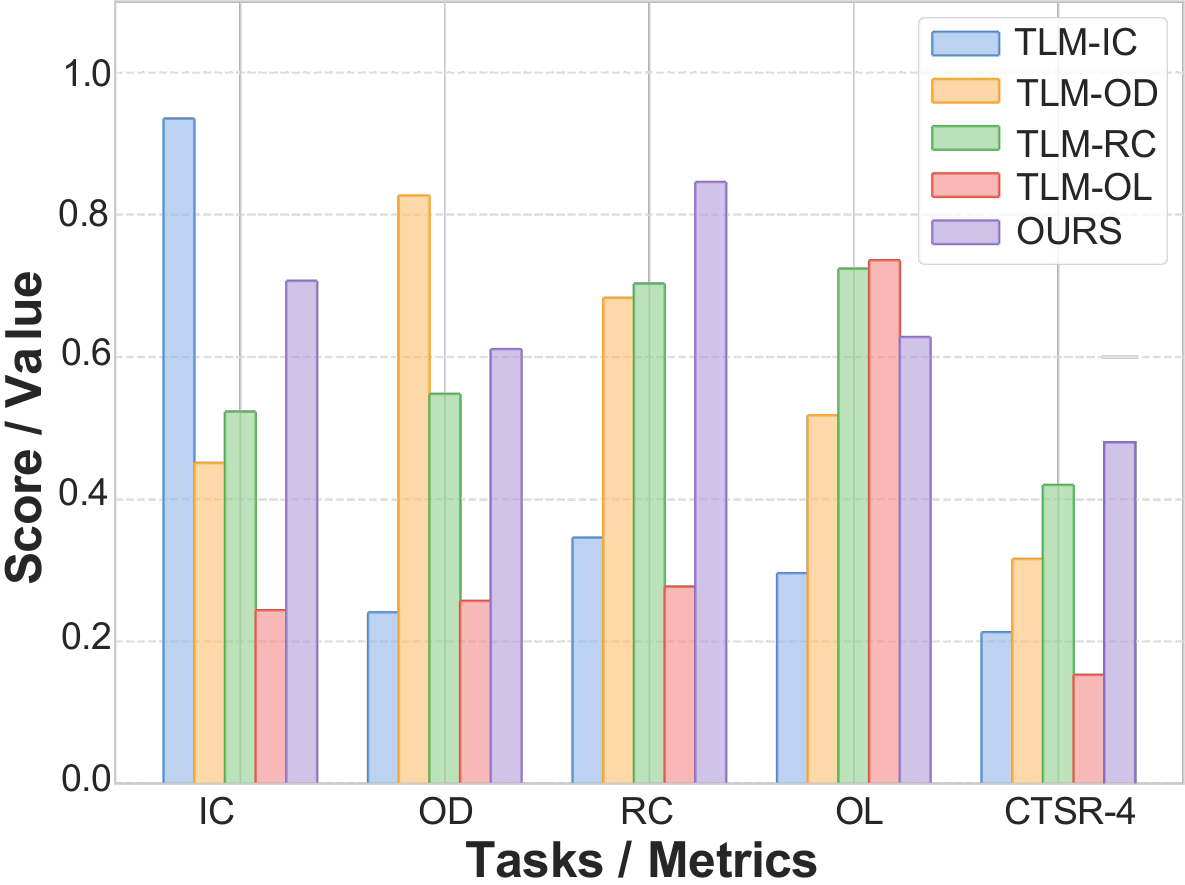}
    \vspace{-0.5cm}
    \caption{Cross-task transferability}
    \label{fig:Mehod_compare}
\end{subfigure}
\hfill
\begin{subfigure}[t]{0.49\linewidth} 
    \centering
    \includegraphics[width=1\linewidth]{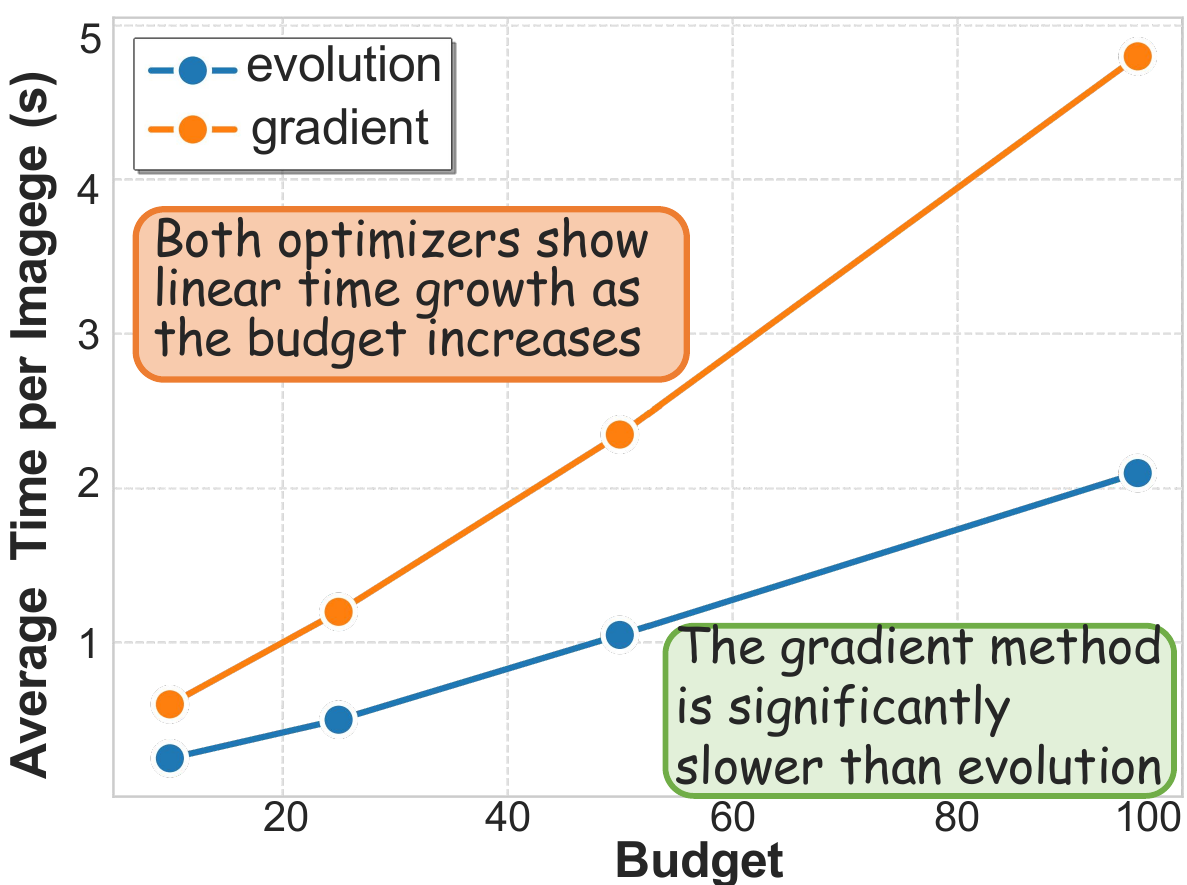}
    \vspace{-0.5cm}
    \caption{Optimization cost}
    \label{fig:cost_compare}
\end{subfigure}
\caption{
Cross-task transferability comparison and Optimization cost.
Figure \ref{fig:Mehod_compare} shows that task-specific TLM attacks perform well on their optimized tasks but have limited cross-task performance, while our method yields more consistent cross-task effects. Figure \ref{fig:cost_compare} reports the average optimization time per image under different search budgets. The gradient-based optimizer is consistently slower than the evolutionary optimizer, with the gap widening as the budget increases.
}
\label{fig:cost}
\vspace{-0.2cm}
\end{figure}

Recent studies on zero-shot adversarial robustness of VLMs have mainly focused on adversarial fine-tuning\cite{mao2022understanding,li2024language,ma2024tima}, especially for CLIP-based\cite{tao2025calibrating} models. For example, FARE\cite{schlarmann2024robust} adversarially fine-tunes the CLIP visual encoder to improve the robustness of downstream LVLMs, while AdvPT\cite{zhang2024adversarial} constructs adversarial image embeddings and aligns them with learnable textual prompts. Despite their promising robustness gains, most existing methods still rely on iterative gradient-based optimization, or even more complex higher-order strategies, to generate strong adversarial examples\cite{jiang2022query,williams2024evolutionary}. This inevitably introduces substantial computational and memory overhead\cite{zhou2024few}, especially when stronger adversaries are pursued through more attack iterations or curvature-related information. Moreover, conventional gradient-based attacks typically optimize a single adversarial candidate along one local gradient trajectory, which restricts the search space exploration and may cause the attack to be trapped in suboptimal adversarial directions, thereby limiting its ability to discover transferable perturbations across tasks~\cite{liu2024gradient,jia2025evolution,li2016seeking}.

To empirically examine these limitations, we conduct comparisons from both optimization cost and cross-task transferability perspectives.
As shown in Fig.~\ref{fig:cost_compare}, the gradient-based optimizer incurs consistently higher average time per image than the evolutionary optimizer under matched search budgets, and the cost gap becomes more evident as the budget increases. Furthermore, Fig.~\ref{fig:Mehod_compare} compares our method with task-specific Training Loss Minimization (TLM) attacks.
Although each TLM variant achieves strong performance on its optimized task, its CTSR-4 score remains much lower than that of our method.
These results suggest that task-specific gradient-based optimization is prone to following task-dependent local trajectories, while our evolutionary search is more effective in discovering cross-task transferable adversarial perturbations with lower optimization cost. The motivation is illustrated in Figure~\ref{fig:motivation}. The detailed experimental results corresponding to Figure~\ref{fig:Mehod_compare} are further provided in Appendix~\ref{supp_sec:additional_results}.

Motivated by this observation, we propose Co-Evolutionary Cross-Modal Attack (CoEvoAttack), a cross-modal adversarial optimization framework guided by evolutionary computation. CoEvoAttack uses evolutionary computation as an outer search process and keeps gradient-based updates for each candidate. On the textual side, we build a constrained population of hard negative samples around the original class. Through selection, mutation, and crossover, these samples stay close to the original semantic neighborhood while providing stronger adversarial pressure to suppress the source category. On the visual side, we keep multiple perturbation candidates within the source object region and update them with momentum-based gradients and evolutionary operations. Therefore, the attack is not limited to one gradient path, but can explore different search directions and select stronger perturbations. This dual evolution strategy reduces optimization cost, improves stability, and helps the attack find stronger negative guidance and more effective visual perturbations.

\begin{figure*}[!t]
    \centering
    \includegraphics[width=0.99\linewidth]{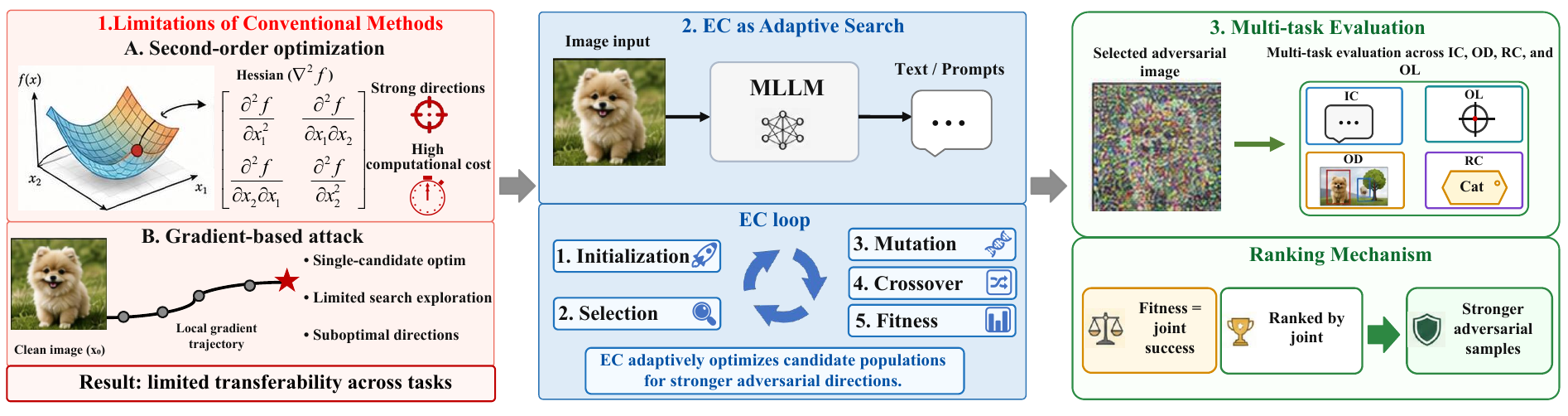}
    \vspace{-0.3cm}
    \caption{
    Motivation of EC-guided multimodal adversarial optimization. Evolutionary computation adaptively searches visual and textual adversarial candidates to overcome the high cost and limited transferability of conventional attacks.
    }
    \label{fig:motivation}
\end{figure*}

From the theoretical perspective, we analyze why the proposed co-evolutionary search can improve cross-task adversarial discovery under constrained object-region perturbations. First, we show that the image-side evolutionary process preserves feasibility: all perturbation candidates remain within the masked object region and satisfy the prescribed $\ell_{\infty}$ perturbation budget through projection and masking operations. In addition, by tracking the best candidate across generations, the best observed adversarial fitness is guaranteed to be non-decreasing, without requiring convexity or smoothness of the VLM loss landscape. We further prove that population-based search increases the probability of reaching high-margin adversarial regions compared with a single randomly initialized trajectory. These analyses provide theoretical support for using evolutionary computation as an adaptive search mechanism to explore multiple feasible adversarial basins and discover more robust and transferable perturbations.

From the experimental perspective, we validate the effectiveness and generality of CoEvoAttack through extensive evaluations on multiple unified vision-language models and cross-task attack settings. Specifically, we conduct experiments on Florence-2, OFA, and UnifiedIO-2 over four representative vision-language tasks, including image captioning, object detection, region categorization, and object localization. The results show that our method achieves stronger average attack success rates and higher cross-task success rates than existing targeted attack baselines, demonstrating that the generated adversarial examples can consistently manipulate shared object-level semantics rather than overfitting to a single task output. Moreover, ablation studies confirm the complementary roles of object-region perturbation, text-side semantic evolution, and image-side perturbation evolution, while analyses on perturbation budget, population size, LoRA adaptation, and evolutionary optimizers provide a comprehensive understanding of the factors affecting attack strength and transferability.

In summary, our main contributions are as follows:
\begin{itemize}
    \item We propose \textbf{CoEvoAttack}, an evolutionary-computation-guided cross-modal attack framework for unified vision-language models. It generates object-level adversarial examples that induce source-to-target semantic shifts across tasks. It jointly evolves textual guidance and localized visual perturbations to explore adversarial directions.

    \item We design a \textbf{dual-sided co-evolutionary optimization mechanism} consisting of text-side and image-side modules. The text-side module evolves hard negative semantic embeddings to strengthen source-category repulsion, while the image-side module evolves object-region perturbation candidates with momentum-guided updates, selection, mutation, and crossover to explore stronger adversarial directions.

    \item We provide \textbf{theoretical analyses} to explain the advantage of the proposed co-evolutionary search. The analyses show that CoEvoAttack preserves feasible object-region perturbations, guarantees non-degradation of the best observed fitness, and improves the probability of discovering high-margin adversarial regions compared with single-trajectory optimization.

    \item We conduct \textbf{extensive experiments and analyses} on multiple unified VLMs and four vision-language tasks. The results demonstrate the strong attack performance and cross-task transferability of CoEvoAttack, while ablation studies and further analyses verify the effectiveness of its key modules and design choices.
\end{itemize}

\section{Related Work}

\subsection{Unified VLMs}

Unified vision-language models (VLMs) align visual and textual representations to support diverse tasks within a shared framework. Representative models, including CLIP~\cite{alayrac2022flamingo}, OFA~\cite{wang2022ofa}, Unified-IO~\cite{lu2024unified}, Florence-2~\cite{xiao2024florence}, and LLaVA-style multimodal assistants~\cite{liu2023visual}, have demonstrated strong generalization across image captioning, visual question answering, object detection, and visual grounding. While such unified semantic representations provide a foundation for multi-task vision-language understanding, they may amplify robustness risks: manipulating shared cross-modal object semantics can cause errors to propagate across multiple tasks. Unlike works that primarily improve unified VLM modeling capability, our study investigates their object-level adversarial vulnerability by examining whether a single adversarial image can consistently shift the perceived semantics of a source object toward a target class across diverse vision-language tasks.

\subsection{Adversarial Attacks}

Adversarial attacks were first studied in image classification, where small perturbations are optimized to mislead neural networks, as in FGSM~\cite{goodfellow2014explaining}, PGD~\cite{MadryMSTV18}, and C\&W~\cite{carlini2017towards}. With the rise of VLMs, attacks have been extended to multimodal systems by perturbing visual inputs, textual inputs, or both to disrupt image-text alignment. Representative methods include Co-Attack~\cite{zhang2022towards}, which jointly optimizes textual and visual adversarial inputs, AttackVLM~\cite{zhao2023evaluating}, which studies transferable attacks on generative VLMs via surrogate models, and CRAFT~\cite{zhao2025one}, which demonstrates that manipulating region-level visual-textual alignment can affect multiple downstream tasks. However, existing attacks largely rely on gradient-based optimization, surrogate-model matching, or task-specific objectives, and mainly search in the visual perturbation space with limited explicit optimization of cross-modal semantic guidance. In contrast, our method performs dual-sided evolutionary search by evolving hard negative textual guidance in the semantic space and object-region perturbations in the visual space, and evaluates whether the same adversarial image can consistently mislead image captioning, object detection, region-to-category recognition, and object location prediction.

\subsection{Evolutionary Computation}

Evolutionary computation (EC) is a family of population-based optimization methods that improve candidate solutions through selection, mutation, crossover, and recombination~\cite{holland1992adaptation,storn1997differential,kennedy1995particle,hansen2006cma,ma2018survey}. Since EC does not require explicit gradients, it is well suited for constrained, discrete, non-differentiable, or black-box optimization problems~\cite{demir2025multi}, and has been applied to adversarial attacks such as the one-pixel attack using differential evolution~\cite{su2019one} and evolution-strategy-based black-box attacks~\cite{qiu2021black}. These studies~\cite{li2022approximated,gupta2015multifactorial,li2025evolutionary,wang2024diversity,li2024multiobjective} demonstrate the flexibility of population-based search as an alternative to purely gradient-based optimization~\cite{basak2012multimodal,cheng2017evolutionary,lin2025landscape,li2023multitask,li2024transfer}. However, existing EC-based attacks mainly focus on conventional vision models or pixel-level perturbation search, with limited exploration of unified VLMs and cross-modal semantic manipulation. In contrast, our framework uses EC as a cross-modal adaptive search mechanism: it evolves hard negative textual embeddings to strengthen semantic adversarial guidance and evolves object-region perturbations with efficient first-order updates, enabling a more stable and efficient attack on shared object-level semantics in unified VLMs.

\begin{figure*}[!t]
    \centering
    \includegraphics[width=0.99\linewidth]{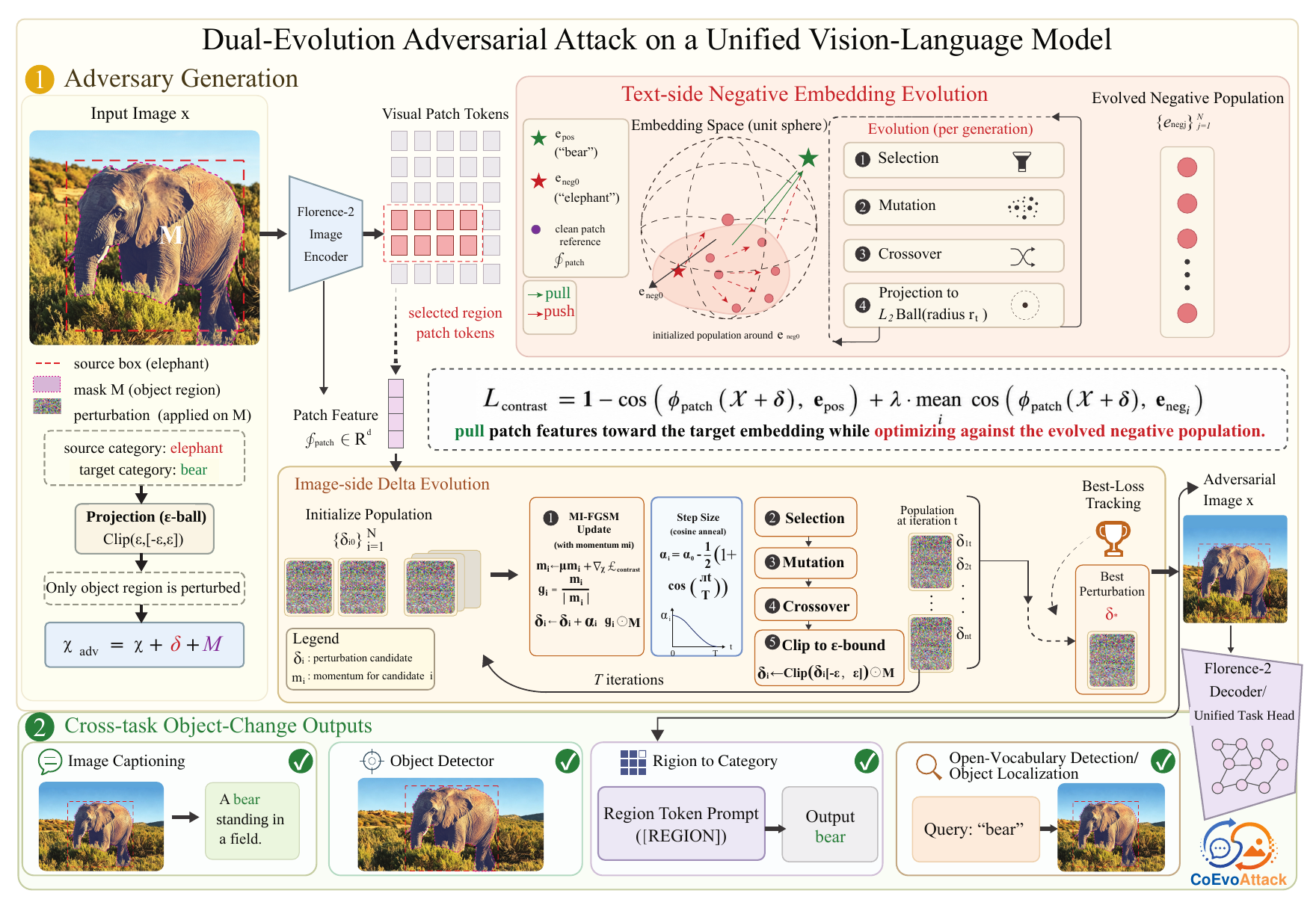}
    \vspace{-0.3cm}
    \caption{
    Overall framework of \textbf{CoEvoAttack}. The method jointly evolves text-side negative embeddings and image-side box-constrained perturbations to generate cross-task adversarial examples for unified vision-language models.
    }
    \label{fig:pipe}
\end{figure*}

\section{Proposed Method}

As illustrated in Fig.~\ref{fig:pipe}, we propose \textbf{CoEvoAttack}, a dual-sided evolutionary adversarial attack framework for unified vision-language models. Given a source object, we first localize it with its bounding box~\cite{zhao2025one} and restrict the perturbation to this region, while keeping all model parameters frozen. On the textual side, \textbf{CoEvoAttack} evolves a population of hard negative embeddings around the source-category semantics through selection, mutation, and crossover, providing stronger semantic repulsion from the original object class. On the image side, it maintains multiple box-constrained perturbation candidates and combines momentum-based first-order updates with evolutionary operations, enabling the attack to explore multiple perturbation trajectories instead of a single gradient path. The optimized perturbation is applied to the source-object region to generate the adversarial image, which is evaluated across image captioning, object detection, region categorization, and object localization to verify its ability to consistently manipulate shared object-level semantics. The following subsections detail the text-side evolution, image-side evolution, and theoretical analysis. A summary of the symbols used in this section is provided in Appendix~\ref{supp_sec:symbols}.

\subsection{Text-side evolutionary computation}

\noindent\textbf{Adversarial Example Initialization.}
On the text side, we first extract an object-level visual reference from the clean image. Given the target patch set $\mathcal T$, the visual representation of the target object is obtained by averaging the visual tokens within the corresponding region:
\begin{equation}
z_v^{clean}=\frac{1}{|\mathcal T|}\sum_{j\in\mathcal T}F_v(x)_j.
\tag{1}
\label{1}
\end{equation}

Here, $F_v(x)_j$ denotes the visual feature produced by the image encoder at the $j$-th target patch, and $\mathcal T$ is obtained by mapping the target bounding box onto the patch grid. The aggregated representation $z_v^{clean}$ serves as the visual reference for subsequent text-side fitness evaluation.

Based on $z_v^{clean}$, we initialize a population of continuous negative samples around the source-category text embedding $e_s^-$. The initial text population is generated as
\begin{equation}
\begin{aligned}
e_k^{(0)}
&=
\Pi_{\mathbb B_2(e_s^-,R_t)}
\big(e_s^-+\xi_k\big), \\
\xi_k
&\sim
\mathcal U(-R_t,R_t),
\qquad
k=1,\dots,K_t .
\end{aligned}
\tag{2}
\label{2}
\end{equation}

Here, $K_t$ is the text population size, $R_t$ denotes the text search radius, and $\Pi_{\mathbb B_2(e_s^-,R_t)}(\cdot)$ projects a candidate onto the $L_2$ ball centered at $e_s^-$ with radius $R_t$. Each candidate is initialized by adding a random continuous perturbation to $e_s^-$ and then projected back into the constrained semantic neighborhood if necessary. This keeps the initial population close to the source-category semantics and avoids unconstrained drift in the text embedding space.

\noindent\textbf{Fitness Evaluation and Selection.}
For text-side evaluation, we compute the clean-state similarity between the source-category embedding and the clean object representation:
\begin{equation}
s_0=\operatorname{sim}(z_v^{clean},e_s^-).
\tag{3}
\label{3}
\end{equation}

The baseline score $s_0$ reflects the original alignment between the clean object representation and the source-category semantics. A good text candidate should preserve this alignment while avoiding collapse toward the target category.

For a text individual $e_k^{(u)}$ in generation $u$, its text-side fitness is defined as
\begin{equation}
\begin{aligned}
\mathcal F_{txt}(e_k^{(u)})
=
-\Big(
&\left|\operatorname{sim}(e_k^{(u)},z_v^{clean})-s_0\right|
+ \\
\|e_k^{(u)}-e_s^-\|_2 
&+
\beta\,\operatorname{sim}(e_k^{(u)},e^+)
\Big).
\end{aligned}
\tag{4}
\label{4}
\end{equation}

The fitness contains three terms: the first term penalizes deviation from the clean-state source-object alignment; the second constrains the candidate around the source-category center; and the third penalizes excessive similarity to the target-category embedding $e^+$, with $\beta$ controlling this target-avoidance constraint. The leading negative sign converts the penalty into a maximization objective, so candidates with smaller semantic deviation, shorter source distance, and lower target similarity obtain higher fitness values. These individuals are therefore treated as higher-quality hard negative samples.

Based on Eq.~\eqref{4}, we retain the top $K_t/3$ elite individuals:
\begin{equation}
\mathcal S^{(u)}=\operatorname{TopK}_{K_t/3}\big(\mathcal E^{(u)};\mathcal F_{txt}\big).
\tag{5}
\label{5}
\end{equation}

Here, $\mathcal E^{(u)}$ denotes the text population at generation $u$, and $\operatorname{TopK}_{K_t/3}(\cdot)$ selects individuals in descending fitness order. This selection preserves candidates that maintain source-semantic consistency, avoid target-category similarity, and provide informative hard negatives for adversarial optimization.

\noindent\textbf{Mutation.}
To explore the local semantic neighborhood of high-quality negative samples, we apply local mutation to each elite individual. For a selected individual $e_k^{sel}\in\mathcal S^{(u)}$, the mutated offspring is generated as
\begin{equation}
\begin{aligned}
e_k^{mut}
&=
\Pi_{\mathbb B_2(e_s^-,R_t)}
\big(
e_k^{sel}+\eta_k
\big), \\
\eta_k
&\sim
\mathcal U(-\phi_tR_t,\phi_tR_t).
\end{aligned}
\tag{6}
\label{6}
\end{equation}

Here, $\phi_t$ denotes the text-side mutation strength, and $\eta_k$ controls the mutation direction and amplitude. The mutation range is scaled by $\phi_tR_t$, where $R_t$ defines the overall text search radius. After perturbation, $\Pi_{\mathbb B_2(e_s^-,R_t)}(\cdot)$ projects the offspring back into the feasible $L_2$ ball centered at $e_s^-$. Thus, mutation performs fine-grained local exploration around high-fitness negative samples, rather than random re-sampling, while preserving source-related semantic locality.

\noindent\textbf{Crossover.}
In addition to mutation, we introduce crossover to recombine the semantic structures of two elite individuals. Given two distinct parents $e_p^{sel}, e_q^{sel}\in\mathcal S^{(u)}$, the crossover offspring is generated by continuous interpolation:
\begin{equation}
\begin{aligned}
e_k^{cross}
&=
\Pi_{\mathbb B_2(e_s^-,R_t)}
\big(
\lambda_k e_p^{sel}+(1-\lambda_k)e_q^{sel}
\big), \\
\lambda_k
&\sim
\mathcal U(0,1).
\end{aligned}
\tag{7}
\label{7}
\end{equation}

Here, $\lambda_k$ determines the contribution of each parent. The offspring lies in the continuous semantic region between two high-fitness negative samples, allowing the population to exploit promising directions discovered by different elites. The projection operator keeps the offspring within the constrained source-semantic neighborhood and prevents semantic drift.

The next-generation text population combines elite individuals, mutated offspring, and crossover offspring:
\begin{equation}
\begin{aligned}
\mathcal E^{(u+1)}
&=
\operatorname{Refill}_{K_t}
\big(
\mathcal S^{(u)}\cup\mathcal M^{(u)}\cup\mathcal C^{(u)}
\big), \\
\mathcal E^-
&=
\mathcal E^{(U)}.
\end{aligned}
\tag{8}
\label{8}
\end{equation}

Here, $\mathcal M^{(u)}$ and $\mathcal C^{(u)}$ denote the mutation and crossover offspring sets at generation $u$, respectively. The operator $\operatorname{Refill}_{K_t}(\cdot)$ maintains a fixed population size $K_t$. After $U$ evolutionary rounds, the final text-side negative-sample population is obtained as $\mathcal E^-=\{e_k^-\}_{k=1}^{K_t}$.

\subsection{Image-Side Evolution Algorithm}

\noindent\textbf{Adversarial Example Initialization.}
On the image side, we initialize a perturbation population within the target-object region. Let $K_v$ denote the image-side population size and $\epsilon$ the perturbation budget. At generation $t=0$, the perturbation and momentum of individual $i$ are initialized as
\begin{equation}
\delta_i^{(0)}\sim\mathcal U(-\epsilon,\epsilon)\odot M,
\qquad
m_i^{(0)}=0,
\quad i=1,\dots,K_v.
\label{9}
\tag{9}
\end{equation}

Here, $M$ is the binary mask of the target-object region, $\delta_i^{(0)}$ is the initialized perturbation, and $m_i^{(0)}$ is the corresponding momentum. The perturbation is sampled within $[-\epsilon,\epsilon]$ and restricted by $M$, forcing pixels outside the target object to remain unchanged. Thus, the image-side adversarial search is confined to the object region instead of the whole image.

The initial image-side population is formulated as
\begin{equation}
\mathcal P^{(0)}=\{(\delta_i^{(0)},m_i^{(0)})\}_{i=1}^{K_v}.
\tag{10}
\label{10}
\end{equation}

Each image-side individual is represented by a perturbation--momentum pair, so the population encodes both current perturbation patterns and their historical update directions, providing a basis for subsequent momentum-guided evolution.

\noindent\textbf{Fitness Evaluation and Selection.}
For the $i$-th perturbation individual at generation $t$, we compute its object-level visual representation by aggregating target-region patch features:
\begin{equation}
z_i^{(t)}
=
\frac{1}{|\mathcal T|}
\sum_{j\in\mathcal T}
F_v(x+\delta_i^{(t)})_j.
\tag{11}
\label{11}
\end{equation}

Here, $\delta_i^{(t)}$ denotes the $i$-th individual’s perturbation. Since only patches in $\mathcal T$ are aggregated, $z_i^{(t)}$ represents the perturbed target-object feature rather than the global representation.

Given the evolved text-side negative population $\mathcal E^-=\{e_k^-\}_{k=1}^{K_t}$, the image-side optimization loss and corresponding fitness are defined as
\begin{equation}
\begin{split}
\mathcal L_i^{(t)}
&=
1-\operatorname{sim}(z_i^{(t)},e^+)
+
\frac{h}{K_t}\sum_{k=1}^{K_t}\operatorname{sim}(z_i^{(t)},e_k^-),\\
\mathcal F_{img}(\delta_i^{(t)})
&=-\mathcal L_i^{(t)}.
\end{split}
\tag{12}
\label{12}
\end{equation}

The first loss term encourages the perturbed object representation to approach the target-category embedding $e^+$, while the second term suppresses its similarity to the evolved hard negative samples, with $h$ controlling the strength of this constraint. Compared with a single fixed negative embedding, this population-level constraint provides more robust semantic repulsion from the source-related neighborhood. Thus, maximizing $\mathcal F_{img}$ is equivalent to minimizing Eq.~\eqref{12}, encouraging stronger target-semantic alignment and weaker association with the evolved hard negatives.

Before population-level selection, each perturbation candidate is individually updated. The gradient normalization and momentum accumulation are written together as
\begin{equation}
\begin{aligned}
\bar g_i^{(t)}
&=
\frac{\nabla_{\delta_i^{(t)}}\mathcal L_i^{(t)}}
{\operatorname{mean}\left(\left|\nabla_{\delta_i^{(t)}}\mathcal L_i^{(t)}\right|\right)+\varepsilon_0},\\
m_i^{(t+1)}
&=\mu m_i^{(t)}+\bar g_i^{(t)}.
\end{aligned}
\tag{13}
\label{13}
\end{equation}

The mean-absolute normalization stabilizes the update scale across individuals, $\varepsilon_0$ avoids numerical instability, and $\mu$ is the momentum decay coefficient. This allows each perturbation candidate to follow an independent MI-FGSM-style trajectory, improving the stability of image-side optimization.

To adaptively control the update magnitude, we use a cosine annealing schedule and update the perturbation as
\begin{equation}
\begin{aligned}
\alpha_t
&=
\alpha\left(
0.1+0.9\cdot\frac{1+\cos(\pi t/T)}{2}
\right),\\
\tilde\delta_i^{(t)}
&=
\Pi_{[-\epsilon,\epsilon]}
\left(
\delta_i^{(t)}-\alpha_t\cdot\operatorname{sign}(m_i^{(t+1)})\odot M
\right).
\end{aligned}
\tag{14}
\label{14}
\end{equation}

The step size gradually decreases, enabling larger exploratory updates in early generations and finer local adjustments later. Here, $\operatorname{sign}(m_i^{(t+1)})$ gives the sign-based descent direction, $M$ restricts the update to the target-object region, and $\Pi_{[-\epsilon,\epsilon]}(\cdot)$ enforces the pixel-wise perturbation budget.

After the individual-level update, the top $K_v/3$ elite candidates are selected according to image-side fitness:
\begin{equation}
\mathcal S^{(t)}
=
\operatorname{TopK}_{K_v/3}
\big(
\{\tilde\delta_i^{(t)}\}_{i=1}^{K_v};
\mathcal F_{img}
\big).
\tag{15}
\label{15}
\end{equation}

Since $\mathcal F_{img}=-\mathcal L_i^{(t)}$, this selection retains the updated perturbations with the smallest image-side loss. The elite set therefore preserves candidates that best promote target-category alignment while reducing similarity to the evolved hard negative semantics.

\noindent\textbf{Mutation.}
For the image-side evolutionary process, we apply local pixel-level mutation to elite perturbation individuals. Given an elite perturbation $\delta_i^{sel}\in\mathcal S^{(t)}$, the mutated offspring is generated as
\begin{equation}
\begin{aligned}
\delta_i^{mut}
&=
\Pi_{[-\epsilon,\epsilon]}
\big(
(\delta_i^{sel}+\zeta_i)\odot M
\big),\\
\zeta_i
&\sim\mathcal U(-\phi_v\epsilon,\phi_v\epsilon).
\end{aligned}
\tag{16}
\label{16}
\end{equation}

Here, $\phi_v$ denotes the image-side mutation strength, and $\zeta_i$ is a local pixel-wise random perturbation. The mutation operates under three constraints: it explores locally around an elite individual, is restricted to the target-object region by the mask $M$, and is clipped by $\Pi_{[-\epsilon,\epsilon]}(\cdot)$ to satisfy the perturbation budget. Therefore, this step performs a locally constrained search around promising perturbation candidates rather than unconstrained sampling over the whole image.

\noindent\textbf{Crossover.}
In addition to mutation, we introduce image-side crossover to recombine the perturbation structures of two elite individuals. Given two distinct elite parents $\delta_p^{\mathrm{sel}}, \delta_q^{\mathrm{sel}} \in S^{(t)}$, their offspring is generated by linear interpolation:
\begin{equation}
\begin{aligned}
\delta_i^{cross}
&=
\Pi_{[-\epsilon,\epsilon]}
\big(
(\lambda_i\delta_p^{sel}+(1-\lambda_i)\delta_q^{sel})\odot M
\big),\\
\lambda_i
&\sim\mathcal U(0,1).
\end{aligned}
\tag{17}
\label{17}
\end{equation}

Here, $\lambda_i$ controls the contribution of each parent perturbation. This operation combines two high-fitness perturbation patterns so that useful local structures discovered by different elites can be inherited by the offspring. The mask $M$ preserves the target-region constraint, and the projection operator ensures that the offspring remains within the perturbation budget.

Finally, the next-generation image population combines elite individuals, mutated offspring, and crossover offspring:
\begingroup
\setlength{\abovedisplayskip}{3pt}
\setlength{\belowdisplayskip}{3pt}
\setlength{\abovedisplayshortskip}{3pt}
\setlength{\belowdisplayshortskip}{3pt}
\begin{equation}
\begin{aligned}
\mathcal P^{(t+1)}
&=
\operatorname{Refill}_{K_v}
\big(
\mathcal S^{(t)}\cup\mathcal M^{(t)}\cup\mathcal C^{(t)}
\big),\\
\delta^*
&=
\arg\max_{\delta\in\cup_{t=0}^{T}\mathcal P^{(t)}}\mathcal F_{img}(\delta),
\qquad
x^{adv}=x+\delta^*.
\end{aligned}
\tag{18}
\label{18}
\end{equation}
\endgroup
\noindent Here, $\mathcal M^{(t)}$ and $\mathcal C^{(t)}$ denote the mutation and crossover offspring sets at generation $t$, respectively, and $\operatorname{Refill}_{K_v}(\cdot)$ maintains a fixed image-side population size $K_v$. After evolution, the final perturbation $\delta^*$ is selected as the individual with the highest image-side fitness over all generations, and the adversarial image is obtained as $x^{adv}=x+\delta^*$.

\begin{table*}[!t]
\vspace{-0.1cm}
\small
\centering
\caption{Comparison of attack performance across four tasks (IC: Image Captioning, OD: Object Detection, RC: Region Categorization, OL: Object Localization) and evaluation metrics (avg: Average Success Rate, CTSR-4: Cross-Task Success Rate on all 4 tasks, CTSR-3: Cross-Task Success Rate on at least 3 tasks).}
\begin{tabular}{ccccccccc}
\hline
\multirow{2}{*}{Models} & \multirow{2}{*}{Method} & \multicolumn{4}{c}{Tasks} & \multicolumn{3}{c}{Evaluate Metrics} \\ \cline{3-9} 
                        &                         & IC & OD & RC & OL & avg & CTSR-4 & CTSR-3 \\ \hline

\multirow{6}{*}{OFA \cite{wang2022ofa}} 
                        & no attack & 0.005 & 0.007 & 0.003 & 0.026 & 0.010 & 0 & 0.001 \\
                        & Mix.Attack \cite{tu2023many} & 0.200 & 0.101 & 0.194 & 0.199 & 0.174 & 0.043 & 0.077 \\
                        & MF-it \cite{zhao2023evaluating} & 0.133 & 0.102 & 0.129 & 0.262 & 0.157 & 0.027 & 0.086 \\
                        & MF-ii \cite{zhao2023evaluating} & 0.176 & 0.150 & 0.154 & 0.295 & 0.194 & 0.067 & 0.098 \\
                        & Attack-Bard \cite{dong2023robust} & \textbf{0.508} & 0.211 & 0.184 & 0.101 & 0.251 & 0.087 & 0.106 \\
                        & CRAFT\cite{zhao2025one} &0.445  &0.324  &0.410  &0.552  &0.443  &0.204  &0.311  \\     
                        \rowcolor{green!10}
                        & \textbf{CoEvoAttack} &0.491  & \textbf{0.355} & \textbf{0.462} & \textbf{0.571}  & \textbf{0.470} & \textbf{0.224} & \textbf{0.342} \\ \hline

\multirow{6}{*}{UnifiedIO-2 \cite{lu2024unified}} 
                        & no attack & 0.002 & 0.002 & 0 & 0.006 & 0.003 & 0 & 0 \\
                        & Mix.Attack \cite{tu2023many} & 0.014 & 0.005 & 0.208 & 0.165 & 0.098 & 0.002 & 0.008 \\
                        & MF-it \cite{zhao2023evaluating} & 0.101 & 0.042 & 0.279 & 0.456 & 0.220 & 0.017 & 0.062 \\
                        & MF-ii \cite{zhao2023evaluating} & 0.471 & 0.146 & 0.323 & 0.347 & 0.322 & 0.090 & 0.283 \\
                        & Attack-Bard \cite{dong2023robust} & \textbf{0.934} & 0.146 & 0.487 & 0.278 & 0.461 & 0.098 & 0.386 \\
                        & CRAFT\cite{zhao2025one} & 0.571 & 0.529 & 0.681 & 0.782 & 0.640 & 0.405 &0.577 \\
                        \rowcolor{green!10}
                        & \textbf{CoEvoAttack} & 0.689  & \textbf{0.620} & \textbf{0.778} &\textbf{0.832}  & \textbf{0.730} & \textbf{0.510} & \textbf{0.689} \\ \hline

\multirow{7}{*}{Florence-2 \cite{xiao2024florence}}  
                        & no attack & 0.002 & 0.001 & 0.005 & 0.009 & 0.004 & 0 & 0 \\
                        & Mix.Attack \cite{tu2023many} & 0.126 & 0.073 & 0.098 & 0.547 & 0.211 & 0.067 & 0.083 \\
                        & MF-it \cite{zhao2023evaluating} & 0.320 & 0.184 & 0.277 & 0.387 & 0.292 & 0.139 & 0.205 \\
                        & MF-ii \cite{zhao2023evaluating} & 0.571 & 0.332 & 0.501 & 0.462 & 0.466 & 0.264 & 0.338 \\
                        & Attack-Bard \cite{dong2023robust} & \textbf{0.935} & 0.241 & 0.346 & 0.296 & 0.308 & 0.213 & 0.310 \\ 
                        & CRAFT\cite{zhao2025one} & 0.530 & 0.491 & 0.655 & \textbf{0.633} & 0.577 & 0.355 & 0.503 \\ 
                        \rowcolor{green!10}
                        & \textbf{CoEvoAttack} & 0.707 & \textbf{0.611} & \textbf{0.846} & 0.628 & \textbf{0.698} & \textbf{0.443} & \textbf{0.623} \\ \hline

\end{tabular}
\label{tab:task_level_attack}
\vspace{-0.2cm}
\end{table*}

\section{Theoretical Analyses}

We provide theoretical analyses of why \textbf{co-evolutionary search improves cross-task adversarial transfer}.

Let $\Delta_M=\{\delta:\|\delta\|_{\infty}\le \epsilon,\ \delta=\delta\odot M\}$ denote the feasible perturbation set induced by the object mask $M$ and the $\ell_\infty$ perturbation budget $\epsilon$. For any $\delta\in\Delta_M$, the object-level visual representation is defined as $z(\delta)=\frac{1}{|\mathcal{T}|}\sum_{j\in\mathcal{T}}F_v(x+\delta)_j$. Given the evolved negative text population $E^-=\{e_k^-\}_{k=1}^{K_t}$, the image-side loss in Eq. (\ref{12}) can be rewritten as $L(\delta)=1-\mathcal{M}(\delta)$, where $\mathcal{M}(\delta)=\operatorname{sim}(z(\delta),e^+)-h\frac{1}{K_t}\sum_{k=1}^{K_t}\operatorname{sim}(z(\delta),e_k^-)$ is the cross-modal adversarial semantic margin. Thus, minimizing $L(\delta)$ is equivalent to maximizing $\mathcal{M}(\delta)$, increasing the target-category alignment while suppressing the source-neighborhood semantics.

\begin{proposition}[Feasibility and Best-Fitness Non-Degradation]
Assume that the image-side evolutionary update applies the projection $\Pi_{[-\epsilon,\epsilon]}(\cdot)$ and the object mask $M$ as in Eqs. (\ref{14})--(\ref{17}). Then every generated perturbation remains in $\Delta_M$. Moreover, if the best individual found so far is retained after each generation, then the best observed image-side fitness is non-decreasing: $\max_{\delta\in \cup_{\tau=0}^{t+1}P^{(\tau)}}F_{\rm img}(\delta)\ge \max_{\delta\in \cup_{\tau=0}^{t}P^{(\tau)}}F_{\rm img}(\delta)$. Equivalently, $\min_{\delta\in \cup_{\tau=0}^{t+1}P^{(\tau)}}L(\delta)\le \min_{\delta\in \cup_{\tau=0}^{t}P^{(\tau)}}L(\delta)$.
\end{proposition}

\begin{proof}
For initialization, each perturbation satisfies $\delta_i^{(0)}\sim U(-\epsilon,\epsilon)\odot M$, and therefore $\delta_i^{(0)}\in\Delta_M$. During the momentum update, mutation, and crossover steps, the perturbation is always masked by $M$ and projected back to $[-\epsilon,\epsilon]$. Hence all offspring also belong to $\Delta_M$, proving feasibility.

For the second claim, the final perturbation is selected according to the best image-side fitness over all generated candidates, as in Eq.~(18). Since the candidate set $\cup_{\tau=0}^{t}P^{(\tau)}$ is a subset of $\cup_{\tau=0}^{t+1}P^{(\tau)}$, taking the maximum over the larger set cannot decrease the best fitness. Because $F_{\rm img}(\delta)=-L(\delta)$, the corresponding best loss is non-increasing.
\end{proof}

\begin{tcolorbox}[width=1.0\linewidth, colframe=blackish, colback=beaublue, boxsep=0mm, arc=2mm, left=2mm, right=2mm, top=5mm, bottom=2mm]
\vspace{-0.3cm}
\noindent\textit{Remark.}
This proposition does not require convexity or smoothness of the VLM loss landscape. It only depends on the projection, masking, and best-candidate tracking mechanisms. Therefore, CoEvoAttack never violates the object-region perturbation constraint and does not discard the best adversarial candidate discovered during the evolutionary process.
\end{tcolorbox}
\vspace{-0.3cm}

\begin{theorem}[Population Search Improves Adversarial Discovery]
Let $\mathcal{A}_{\tau}=\{\delta\in\Delta_M:\mathcal{M}(\delta)\ge \tau\}$ be the set of perturbations whose adversarial semantic margin is at least $\tau$. Suppose a single randomly initialized feasible candidate falls into $\mathcal{A}_{\tau}$ with probability $p_\tau$. With $K_v$ independently initialized image-side candidates, the probability that at least one candidate reaches $\mathcal{A}_{\tau}$ is $\mathbb{P}\left(\exists i\in[K_v],\delta_i^{(0)}\in\mathcal{A}_{\tau}\right)=1-(1-p_\tau)^{K_v}$. Consequently, $1-(1-p_\tau)^{K_v}\ge p_\tau$, with strict improvement whenever $K_v>1$ and $0<p_\tau<1$.
\end{theorem}

\begin{proof}
Each candidate independently misses the successful region $\mathcal{A}_{\tau}$ with probability $1-p_\tau$. Therefore, all $K_v$ candidates miss $\mathcal{A}_{\tau}$ with probability $(1-p_\tau)^{K_v}$. Taking the complement gives the probability that at least one candidate enters $\mathcal{A}_{\tau}$: $1-(1-p_\tau)^{K_v}$. Since $(1-p_\tau)^{K_v}\le 1-p_\tau$ for $K_v\ge 1$, we have $1-(1-p_\tau)^{K_v}\ge p_\tau$. The inequality is strict when $K_v>1$ and $p_\tau\in(0,1)$.
\end{proof}

\begin{tcolorbox}[width=1.0\linewidth, colframe=blackish, colback=beaublue, boxsep=0mm, arc=2mm, left=2mm, right=2mm, top=5mm, bottom=2mm]
\vspace{-0.3cm}
\noindent\textit{Remark.}
This theorem explains why population-based perturbation evolution can discover stronger transferable attacks than single-trajectory gradient optimization. A conventional gradient attack follows one local path from one initialization, whereas CoEvoAttack explores multiple feasible object-region perturbation basins in parallel. The theorem intentionally avoids claiming global optimality; instead, it establishes a basin-hitting advantage, which is more appropriate for non-convex VLM attack landscapes.
\end{tcolorbox}
\vspace{-0.3cm}


\begin{table*}[t]
\vspace{-0.1cm}
\small
\centering
\caption{Ablation study on perturbation locality and text-/image-side evolutionary computation across four vision-language tasks and cross-task metrics.}
\setlength{\tabcolsep}{5pt}
\renewcommand{\arraystretch}{1.1}
\begin{tabular}{lccc|cccc|ccc}
\toprule
\multirow{2}{*}{Method}
& \multirow{2}{*}{Box}
& \multirow{2}{*}{Text-side}
& \multirow{2}{*}{Image-side}
& \multicolumn{4}{c|}{Tasks}
& \multicolumn{3}{c}{Evaluation Metrics} \\
\cmidrule(lr){5-8} \cmidrule(lr){9-11}
& & & & IC & OD & RC & OL & Avg & CTSR-4 & CTSR-3 \\
\midrule

PGD
&  &  &  
& 0.625 & 0.343 & 0.776 & 0.431
& 0.544 & 0.246 & 0.385 \\

PGD+Box
& $\checkmark$ &  &  
& 0.601 & 0.493 & 0.720 & 0.493
& 0.563 & 0.311 & 0.444 \\

TEXT+EC
& $\checkmark$ & $\checkmark$ &  
& 0.609 & 0.568 & 0.718 & \textbf{0.709}
& 0.651 & 0.440 & 0.588 \\

IMAGE+EC
& $\checkmark$ &  & $\checkmark$
& 0.620 & 0.573 & 0.771 & 0.648
& 0.653 & 0.414 & 0.585 \\

EC w/o Box
&  & $\checkmark$ & $\checkmark$
& \textbf{0.790} & 0.531 & 0.806 & 0.526
& 0.685 & 0.385 & 0.556 \\

\rowcolor{green!10}
\textbf{CoEvoAttack}
& $\checkmark$ & $\checkmark$ & $\checkmark$
& 0.707 & \textbf{0.611} & \textbf{0.846} & 0.628
& \textbf{0.698} & \textbf{0.443} & \textbf{0.623} \\
\bottomrule
\end{tabular}
\small
\label{tab:ablation}
\vspace{-0.2cm}
\end{table*}

\section{Experiment}

In this section, we empirically evaluate the effectiveness of \textbf{CoEvoAttack} under the cross-task object-change attack setting. We first describe the experimental setup, including the attacked unified VLMs, the benchmark dataset, the evaluated tasks, the metrics, and the implementation details. We then compare \textbf{CoEvoAttack} with existing targeted attack methods on different unified VLMs to demonstrate its overall attack performance and cross-task transferability. Next, we conduct ablation studies to verify the contributions of the box-constrained perturbation, text-side evolutionary computation, and image-side evolutionary computation. Finally, we provide further analyses on perturbation budget, population size, parameter-efficient adaptation, and image-side evolutionary optimizer, giving a more complete understanding of the key factors that affect the performance of \textbf{CoEvoAttack}.

\subsection{Experiment Setup}

We evaluate our method on three representative unified vision-language models, including Florence-2-Large~\cite{xiao2024florence}, UnifiedIO-2-Large~\cite{lu2024unified}, and OFA-Large~\cite{wang2022ofa}, using their original pre-trained weights without task-specific fine-tuning. Experiments are conducted on the CrossVLAD benchmark~\cite{zhao2025one}, which contains 3,000 MSCOCO images with 79 object-change pairs across 10 semantic categories. More details of the dataset are provided in Appendix~\ref{supp_sec:dataset}. Following the benchmark protocol, we evaluate targeted object-change attacks on four tasks: image captioning, object detection, region categorization, and object localization. Adversarial examples are generated on Florence-2-Large under an $\ell_{\infty}$ constraint with $\epsilon = 32/255$ for 100 iterations using a step size of $\alpha = 8/255$. We report CTSR-4, CTSR-3, and the average success rate across the four tasks. All experiments are conducted on a dual-vGPU-48GB platform. Further implementation details are in Appendix~\ref{supp_sec:experiment}.

\subsection{Main Results}
Table~\ref{tab:task_level_attack} reports the attack performance of different methods across four vision-language tasks and three unified VLMs, including Florence-2~\cite{xiao2024florence}, OFA~\cite{wang2022ofa}, and UnifiedIO-2~\cite{lu2024unified}. Overall, our \textbf{CoEvoAttack} achieves the strongest performance among all compared methods, obtaining the best average attack success rate and the highest cross-task success rates. In particular, the consistent improvements on CTSR-4 and CTSR-3 demonstrate that \textbf{CoEvoAttack} can generate adversarial images that simultaneously mislead multiple tasks, rather than merely exploiting a single task-specific output. Although some baseline methods achieve competitive results on individual tasks, their performance drops markedly under the stricter cross-task evaluation, indicating limited ability to manipulate shared object-level semantics. In contrast, \textbf{CoEvoAttack} maintains stronger cross-task consistency by jointly evolving textual semantic guidance and visual perturbation candidates. These results suggest that the proposed evolutionary-computation-guided optimization induces a more reliable semantic shift from the source object to the target category, leading to adversarial examples with stronger task transferability and semantic consistency.

\begin{figure}[!t]
\vspace{0.1cm}
\begin{subfigure}[t]{0.49\linewidth} 
    \centering
    \includegraphics[width=1\linewidth]{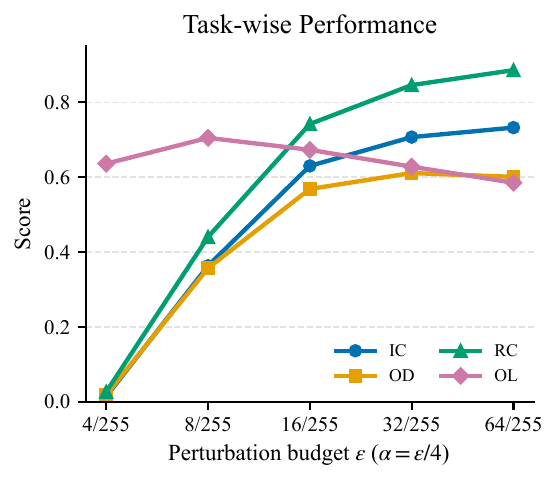}
    \vspace{-0.5cm}
    \caption{Task-wise performance}
    \label{fig:parameter_task}
\end{subfigure}
\hfill
\begin{subfigure}[t]{0.49\linewidth} 
    \centering
    \includegraphics[width=1\linewidth]{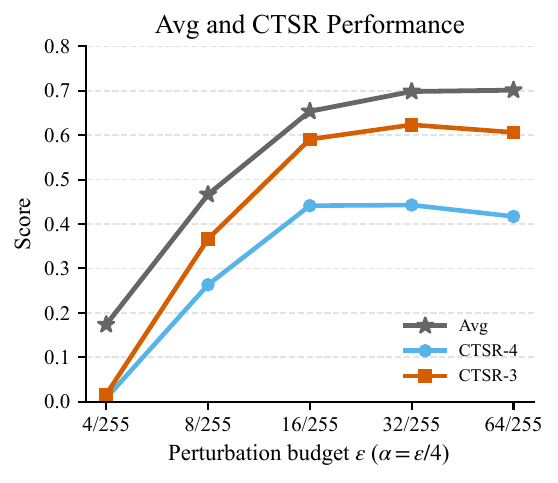}
    \vspace{-0.5cm}
    \caption{Overall performance}
    \label{fig:parameter_avg}
\end{subfigure}
\caption{
Parameter analysis under different perturbation budgets.
In Figure~\ref{fig:parameter_task}, we report the task-wise performance on image captioning (IC), object detection (OD), region categorization (RC), and object localization (OL).
In Figure~\ref{fig:parameter_avg}, we report the average score and cross-task success rates, including CTSR-4 and CTSR-3.
The step size is set to $\alpha=\epsilon/4$ for all settings.
}
\label{fig:parameter}
\vspace{-0.2cm}
\end{figure}

\newcommand{\bestcell}[1]{\cellcolor{green!10}#1}

\begin{table}[!t]
\vspace{0.05cm}
\centering
\caption{Effect of population size in the evolutionary algorithm on image-side and text-side optimization. $P_T$ and $P_I$ denote the population sizes of the text-side and image-side evolutionary algorithms, respectively.}
\small
\setlength{\tabcolsep}{3pt}
\renewcommand{\arraystretch}{1.05}
\resizebox{\columnwidth}{!}{
\begin{tabular}{llccccccc}
\toprule
\multirow{2}{*}{Fixed Side} 
& \multirow{2}{*}{Varied Side} 
& \multicolumn{4}{c}{Task Metrics} 
& \multicolumn{3}{c}{Overall Metrics} \\ 
\cmidrule(lr){3-6} \cmidrule(lr){7-9}
& & IC & OD & RC & OL & Avg. & CTSR-4 & CTSR-3 \\
\midrule

\multirow{4}{*}{$P_T=9$}
& $P_I=6$  & 0.701 & 0.598 & 0.844 & 0.621 & 0.691 & 0.432 & 0.610 \\
& \bestcell{$P_I=9$}  
& \bestcell{\textbf{0.707}} 
& \bestcell{\textbf{0.611}} 
& \bestcell{0.846} 
& \bestcell{0.628} 
& \bestcell{\textbf{0.698}} 
& \bestcell{\textbf{0.443}} 
& \bestcell{\textbf{0.623}} \\
& $P_I=12$ & 0.705 & 0.599 & 0.846 & 0.617 & 0.692 & 0.435 & 0.610 \\
& $P_I=15$ & 0.704 & 0.595 & \textbf{0.849} & \textbf{0.629} & 0.694 & 0.437 & 0.612 \\
\midrule

\multirow{4}{*}{$P_I=9$}
& $P_T=6$  & \textbf{0.709} & 0.601 & 0.845 & 0.621 & 0.694 & 0.435 & 0.614 \\
& \bestcell{$P_T=9$}  
& \bestcell{0.707} 
& \bestcell{\textbf{0.611}} 
& \bestcell{\textbf{0.846}} 
& \bestcell{0.628} 
& \bestcell{\textbf{0.698}} 
& \bestcell{\textbf{0.443}} 
& \bestcell{\textbf{0.623}} \\
& $P_T=12$ & 0.705 & 0.600 & \textbf{0.846} & \textbf{0.630} & 0.695 & 0.439 & 0.615 \\
& $P_T=15$ & 0.702 & 0.595 & 0.842 & 0.624 & 0.691 & 0.429 & 0.612 \\
\bottomrule
\end{tabular}
}
\label{tab:population}
\vspace{-0.2cm}
\end{table}

\subsection{Ablation Study}

We conduct ablation studies to verify the effectiveness of the main components in \textbf{CoEvoAttack}, including the object-region constraint, text-side evolutionary computation, and image-side evolutionary computation. As shown in Table~\ref{tab:ablation}, the vanilla PGD~\cite{MadryMSTV18} baseline achieves limited overall and cross-task attack performance, suggesting that single-trajectory perturbation optimization is insufficient for transferable adversarial attacks. Introducing the object-region constraint improves cross-task metrics, indicating the benefit of focusing perturbations on the source object. Based on this localized setting, both TEXT+EC and IMAGE+EC further enhance performance: TEXT+EC strengthens semantic repulsion from the source category through evolved hard negative guidance, while IMAGE+EC explores multiple visual perturbation candidates beyond a single gradient path. Although EC w/o Box performs well on some task-wise metrics, its cross-task performance remains inferior to the full model. By integrating all three components, \textbf{CoEvoAttack} achieves the best average success rate and cross-task success rates, demonstrating their complementary roles in generating adversarial examples.

Fig.~\ref{fig:parameter} analyzes the effect of the perturbation budget $\epsilon$ on attack performance. When $\epsilon$ is small, the feasible perturbation space is limited, making it difficult to shift the source object toward the target category. As $\epsilon$ increases, both task-wise and overall metrics improve, indicating that a larger budget allows the evolutionary search to discover stronger adversarial candidates. However, the performance gradually saturates under larger budgets, and cross-task metrics no longer improve consistently, suggesting that simply increasing perturbation magnitude cannot continuously enhance semantic consistency. Therefore, we set $\epsilon=32/255$ in the main experiments to balance attack strength and cross-task transferability. The detailed results corresponding to Fig.~\ref{fig:parameter} are provided in Appendix~\ref{supp_sec:additional_results}.

\begin{figure*}[!t]
    \centering
    \includegraphics[width=0.99\linewidth]{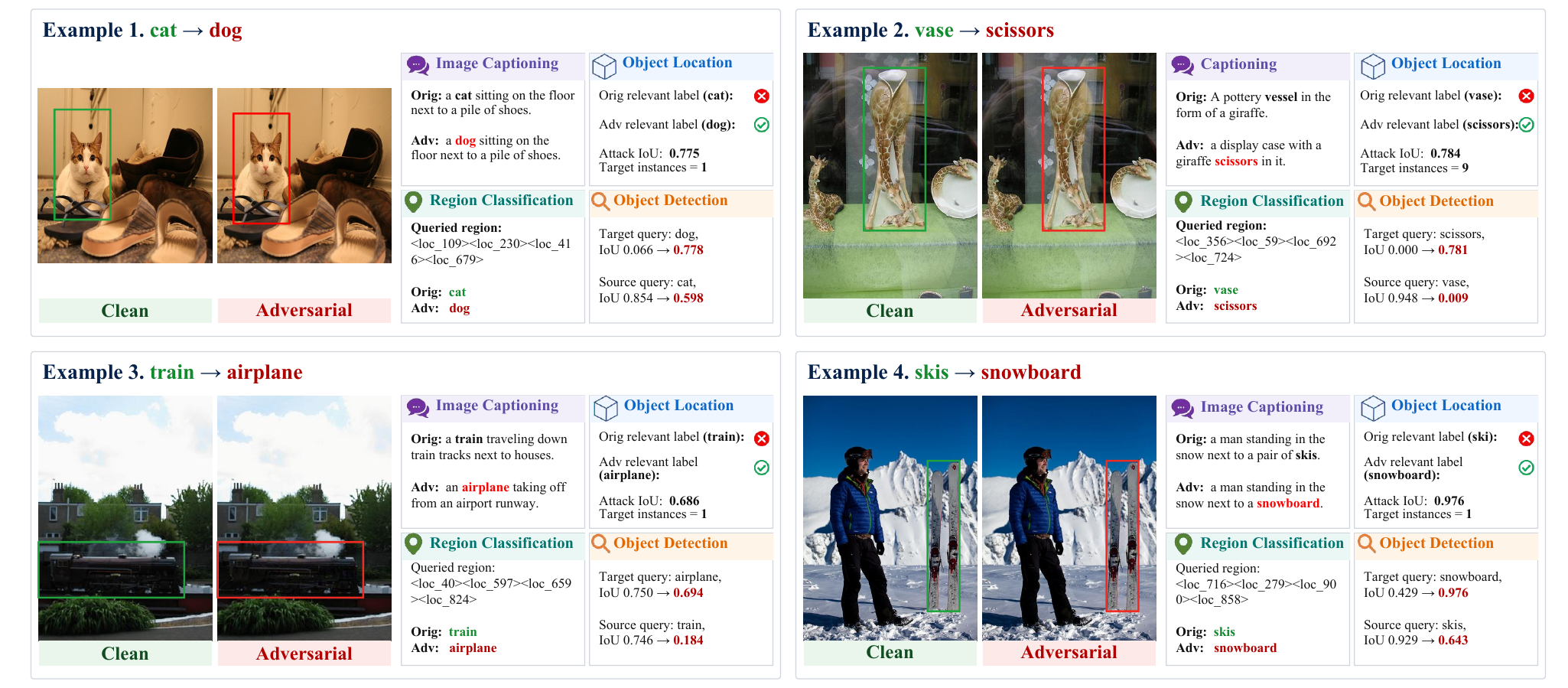}
    \vspace{-0.3cm}
    \caption{
    Cross-task transfer of adversarial examples on Florence-2. Each example compares a clean image with its adversarial counterpart across four vision-language tasks: image captioning, object localization, region classification, and object detection. Green denotes the source concept and red denotes the target concept.IoU denotes intersection over union. More visualization examples are provided in Section~\ref{supp_sec:visualization}.
    }
    \label{fig:instance_panel}
\end{figure*}

\begin{table}[!t]
\vspace{-0.1cm}
\caption{Comparison of PGD, full \textbf{CoEvoAttack}, and its LoRA-based variant. LoRA updates only 0.81\% of parameters while maintaining competitive attack performance.}
\resizebox{\linewidth}{!}{
\begin{tabular}{l  c c c c c c c}
\toprule
\multirow{2}{*}{Method} & \multicolumn{4}{c}{Tasks} & \multicolumn{3}{c}{Metric} \\
\cmidrule(lr){2-5} \cmidrule(lr){6-8}
  & IC & OD & RC & OL & avg & CTSR-4 & CTSR-3 \\
\midrule
PGD
& 0.625 & 0.343 & 0.776 & 0.431
& 0.544 & 0.246 & 0.385 \\
LoRA 
& 0.526 & 0.574 & 0.845 & \textbf{0.634} 
& 0.645 & 0.274 & 0.588 \\
\rowcolor{green!10}
\textbf{CoEvoAttack } 
& \textbf{0.707} & \textbf{0.611} & \textbf{0.846} & 0.628 
& \textbf{0.698} & \textbf{0.443} & \textbf{0.623} \\
\bottomrule
\end{tabular}
}
\centering
\label{tab:lora_results}
\vspace{-0.1cm}
\end{table}

\begin{table}[!t]
\vspace{-0.1cm}
\caption{Comparison of different image-side evolutionary optimizers. DE, PSO, and RCGA are used to optimize the image-side perturbation population under the same framework.}
\resizebox{\linewidth}{!}{
\begin{tabular}{l  c c c c c c c}
\toprule
\multirow{2}{*}{EC-strategy} & \multicolumn{4}{c}{Tasks} & \multicolumn{3}{c}{Metric} \\
\cmidrule(lr){2-5} \cmidrule(lr){6-8}
  & IC & OD & RC & OL & avg & CTSR-4 & CTSR-3 \\
\midrule

DE  
& \textbf{0.794} & 0.543 & \textbf{0.904} & 0.531
& 0.693 & 0.391 & 0.561 \\
PSO 
& 0.571 & 0.533 & 0.687 & \textbf{0.690} 
& 0.621 & 0.411 & 0.557 \\
\rowcolor{green!10}
\textbf{RCGA}  
& 0.707 & \textbf{0.611} & 0.846 & 0.628 
& \textbf{0.698} & \textbf{0.443} & \textbf{0.623} \\
\bottomrule
\end{tabular}
}
\centering
\label{tab:ec}
\vspace{-0.1cm}
\end{table}

\subsection{Further Analyses}

The effect of population size on text-side and image-side evolutionary optimization is reported in Table~\ref{tab:population}. According to the empirical results, the best overall performance is obtained when the text-side and image-side population sizes are both set to 9. Either increasing or decreasing the population size does not lead to better results. A smaller population may provide insufficient candidate diversity, limiting the search space explored by the evolutionary process. In contrast, an excessively large population may introduce redundant candidates and less informative search directions, which do not further improve the attack performance. Therefore, setting $P_T=P_I=9$ provides a better balance between candidate diversity, optimization stability, and attack effectiveness, and we adopt this setting in the main experiments.

We incorporate LoRA modules into the original \textbf{CoEvoAttack} framework to construct a parameter-efficient variant. In Table~\ref{tab:lora_results}, we compare PGD~\cite{MadryMSTV18}, the full version of \textbf{CoEvoAttack}, and its LoRA-enhanced variant. Specifically, this LoRA-based variant updates only 0.81\% of the model parameters, corresponding to a 99.19\% reduction in trainable parameters compared with full fine-tuning. Although its performance remains slightly below that of the full \textbf{CoEvoAttack}, it consistently outperforms PGD~\cite{MadryMSTV18} across most evaluation metrics. These results demonstrate that incorporating parameter-efficient adaptation into \textbf{CoEvoAttack} can effectively strengthen adversarial attacks while requiring substantially fewer trainable parameters. Given its significant reduction in optimization cost, the LoRA-based variant offers a practical and resource-efficient alternative when computational resources are limited.

We further conduct experiments with different evolutionary algorithms on the image side under the same attack framework, as shown in Table~\ref{tab:ec}. Specifically, while keeping the other components unchanged, we use DE~\cite{storn1997differential}, PSO~\cite{kennedy1995particle}, and RCGA~\cite{herrera1998tackling} to optimize the image-side perturbation population. The results show that different evolutionary algorithms exhibit different strengths on individual tasks. Nevertheless, RCGA achieves the best overall performance in terms of the average success rate and cross-task success rates. This indicates that RCGA is more effective in balancing task-wise attack strength and cross-task adversarial consistency. Therefore, we adopt RCGA as the image-side evolutionary optimizer in \textbf{CoEvoAttack}.

\section{Conclusion}

In this paper, we presented CoEvoAttack, an evolutionary-computation-guided cross-modal adversarial attack framework for unified vision-language models. CoEvoAttack jointly performs text-side semantic evolution and image-side perturbation evolution, rather than merely increasing pixel-level perturbation strength. On the text side, it evolves source-neighborhood negative embeddings into harder semantic guidance to suppress the original object category. On the image side, it maintains multiple object-region perturbation candidates and refines them with momentum-guided updates and evolutionary operations, enabling exploration beyond a single gradient path. By coupling these two searches, CoEvoAttack improves optimization efficiency, stability, and cross-task transferability. Theoretical analysis further shows that the dual-sided evolutionary mechanism provides a population-level surrogate for exploring the cross-modal adversarial search space. Extensive experiments on multiple unified VLMs and vision-language tasks demonstrate that CoEvoAttack achieves strong, efficient, and task-transferable attack performance.

\bibliographystyle{IEEEtran}
\bibliography{main}

\clearpage
\raggedbottom

\setcounter{section}{0}
\setcounter{subsection}{0}
\setcounter{figure}{0}
\setcounter{table}{0}
\setcounter{equation}{0}

\renewcommand{\thefigure}{S\arabic{figure}}
\renewcommand{\thetable}{S\arabic{table}}
\renewcommand{\theequation}{S\arabic{equation}}

\makeatletter
\twocolumn[
\begin{@twocolumnfalse}
\begin{center}

{\Huge
Two Sides of the Same Coin: Co-Evolving Search\\
for Cross-Task Attacks on Vision-Language Models\\
(Supplementary Material)
\par}

 \vspace{1.00em}

 {\sublargesize
 Xuanhui Lin,
 Junhao Dong,
 Mingrong Gong,
 Yucheng Chen,
 Xinghua Qu and
 Yew-Soon~Ong,~\emph{Fellow, IEEE}
 \par}

\vspace{0.8em}

\end{center}
\end{@twocolumnfalse}
]
\makeatother

%
%
%
%

{\centering
\textsc{Appendices}
\par}
\renewcommand{\thesection}{\Alph{section}}
\renewcommand{\thesubsection}{\thesection.\arabic{subsection}}
\vspace{0.4em}

\noindent\textbf{Appendix Overview.}
This appendix provides supplementary details for the proposed \textbf{CoEvoAttack} framework. Section~\ref{supp_sec:positioning} clarifies the positioning of our method relative to existing task-specific and gradient-based attacks, highlighting the difference between single-trajectory optimization and our cross-modal co-evolutionary search. Section~\ref{supp_sec:dataset} introduces the CrossVLAD dataset configuration and the evaluated object-change tasks. Section~\ref{supp_sec:experiment} describes the experimental settings, implementation details, and evaluation protocol used for cross-task attack assessment. Section~\ref{supp_sec:additional_results} reports additional numerical results, including the task-specific TLM comparison and the detailed perturbation-budget analysis corresponding to Fig.~\ref{fig:parameter}. Section~\ref{supp_sec:symbols} summarizes the key notation used in our method. Finally, the theoretical analysis explains how object-region localization is implemented through region-token projection and mask-constrained perturbation optimization.

\section{Positioning of Our Work}
\label{supp_sec:positioning}

Existing adversarial attacks on vision-language models usually optimize a single adversarial candidate along a task-specific gradient trajectory~\citesupp{sup_zhang2022towards,sup_zhao2023evaluating}, following the common first-order adversarial optimization paradigm established by PGD-style attacks~\citesupp{madry2017towards}. While effective for a particular task or prompt, such optimization may overfit to task-dependent local directions and provide limited cross-task transferability. In contrast, \textbf{CoEvoAttack} formulates the attack as a dual-sided evolutionary search problem. On the text side, it evolves hard negative semantic embeddings around the source-category representation to provide stronger cross-modal repulsion. On the image side, it maintains a population of object-region perturbations and combines momentum-based first-order updates with selection, mutation, and crossover. This design enables the attack to explore multiple feasible adversarial candidates rather than relying on a single gradient path.

From a methodological perspective, the closest line of work is targeted adversarial attacking for unified vision-language models. However, most prior methods mainly search in the visual perturbation space or optimize task-specific objectives. Our method instead performs co-evolutionary search in both the semantic guidance space and the image perturbation space. This allows the generated adversarial image to induce a consistent semantic shift from the source object to the target category across multiple downstream tasks.

\section{Dataset Configurations}
\label{supp_sec:dataset}

We conduct experiments on the CrossVLAD benchmark~\citesupp{sup_zhao2025one}, 
which is specifically designed for evaluating cross-task adversarial attacks 
on unified vision-language models. CrossVLAD is constructed from the 
MSCOCO train2017 dataset~\citesupp{sup_lin2014microsoft} and contains 3,000 carefully selected images 
with 79 object-change pairs across 10 semantic categories, including 
vehicle, outdoor, animal, accessory, sports, kitchen, food, furniture, 
electronic, and appliance. Each sample is associated with a source object 
category and a semantically reasonable target category, enabling targeted 
object-change evaluation across multiple vision-language tasks. Following 
the original benchmark setting, we evaluate attacks on four representative 
tasks: image captioning, object detection, region categorization, and object 
localization. The dataset construction follows strict filtering criteria, 
including object size constraints, limited object instances, category 
uniqueness, caption verification, and exclusion of images containing the 
target category. Original MSCOCO annotations are preserved, while 
GPT-4-assisted annotations are used to generate target-category captions, 
ensuring that the generated descriptions explicitly mention the target 
category while excluding the source category.


\section{Experiment Details}
\label{supp_sec:experiment}

\noindent\textbf{Implementation Details.}
In our experiments, we use Florence-2-Large~\citesupp{sup_xiao2024florence} as the target unified
vision-language model. We generate adversarial examples under an
$\ell_{\infty}$ constraint with $\epsilon = 32/255$ and optimize the
perturbations for 100 iterations with a step size of $\alpha = 8/255$.
Perturbations are restricted to the object mask. For the evolved negative
prompting, we use a text-side negative population of 9 with a mutation
factor of $\phi_t = 0.1$. On the image side, we maintain a perturbation
population of 9 with $\phi_i = 0.1$, and update each candidate using a
momentum-based iterative rule~\citesupp{dong2018boosting} together with selection, mutation, and crossover. The negative similarity weight is set to 0.2, and the negative
population is evolved for 20 steps with an $\ell_2$ radius of 0.05. We
evaluate attack performance on four tasks, including image captioning,
object detection, region categorization, and object localization. Following
the benchmark protocol, we report CTSR-4 and CTSR-3, which measure the
proportion of samples that successfully attack all four tasks or at least
three tasks, respectively, and also report the average success rate across
the four tasks. All experiments are conducted on a dual-vGPU-48GB platform.

Unless otherwise specified, we use Florence-2-Large as the main target model and evaluate transferability on OFA-Large and UnifiedIO-2-Large. The adversarial perturbation is optimized under an $\ell_{\infty}$ constraint with perturbation budget $\epsilon=32/255$. Perturbations are restricted to the object region using the binary mask $M$. The text-side population size and image-side population size are both set to $9$ by default. The text-side evolution searches for hard negative embeddings within an $L_2$ ball centered at the source-category embedding, while the image-side evolution maintains multiple perturbation candidates and updates them with momentum-guided optimization together with mutation and crossover.

\noindent\textbf{Evaluation Protocol.}
We follow the CrossVLAD evaluation setting for cross-task object-change attacks. Each adversarial example is evaluated on four complementary vision-language tasks: image captioning (IC), object detection (OD), region categorization (RC), and object localization (OL). We evaluate our method on three representative unified vision-language models, including Florence-2-Large~\citesupp{sup_xiao2024florence},
UnifiedIO-2-Large~\citesupp{sup_lu2024unified}, and
OFA-Large~\citesupp{sup_wang2022ofa}. All models are used with their original pre-trained weights without any task-specific fine-tuning, allowing us to assess the impact of adversarial attacks on pre-trained unified models. Besides task-wise attack success rates, we report the average success rate, CTSR-4, and CTSR-3. CTSR-4 measures the proportion of samples that successfully attack all four tasks simultaneously, while CTSR-3 measures the proportion of samples that successfully attack at least three tasks. These metrics are stricter than single-task attack success rates because they evaluate whether the perturbation consistently affects the shared object-level semantic representation across tasks.

\begin{table}[!t]
\centering
\caption{Performance comparison of CTSR-4 between task-specific Training Loss Minimization (TLM) attacks and our method. TLM-IC, TLM-OD, TLM-RC, and TLM-OL denote attacks optimized for Image Captioning, Object Detection, Region Categorization, and Object Localization, respectively. Bold values indicate the best result in each column.}
\small
\setlength{\tabcolsep}{4pt}
\renewcommand{\arraystretch}{1.05}
\resizebox{0.96\columnwidth}{!}{
\begin{tabular}{lccccc}
\toprule
\multirow{2}{*}{Method} & \multicolumn{4}{c}{Tasks} & Metric \\
\cmidrule(lr){2-5} \cmidrule(lr){6-6}
& IC & OD & RC & OL & CTSR-4 \\
\midrule
TLM-IC & \textbf{0.935} & 0.241 & 0.346 & 0.296 & 0.213 \\
TLM-OD & 0.451 & \textbf{0.827} & 0.683 & 0.518 & 0.316 \\
TLM-RC & 0.523 & 0.548 & 0.703 & 0.724 & 0.420 \\
TLM-OL & 0.244 & 0.257 & 0.277 & \textbf{0.736} & 0.153 \\
\rowcolor{green!10}
CoEvoAttack   & 0.707 & 0.611 & \textbf{0.846} & 0.628 & \textbf{0.623} \\
\bottomrule
\end{tabular}
}
\label{tab:tlm_compare}
\end{table}

\begin{table}[t]
\centering
\caption{Effect of perturbation budget $\epsilon$ on task-wise and overall performance.
The attack step size is set as $\alpha=\epsilon/4$.
The selected setting $\epsilon=32/255$ is highlighted.}
\small
\setlength{\tabcolsep}{4pt}
\renewcommand{\arraystretch}{1.15}
\resizebox{\columnwidth}{!}{
\begin{tabular}{lccccccc}
\toprule
\multirow{2}{*}{Perturbation} 
& \multicolumn{4}{c}{Task Metrics} 
& \multicolumn{3}{c}{Overall Metrics} \\ 
\cmidrule(lr){2-5} \cmidrule(lr){6-8}
& IC & OD & RC & OL & Avg. & CTSR-4 & CTSR-3 \\
\midrule

$\epsilon=4/255$  
& 0.014 & 0.018 & 0.026 & 0.636 & 0.174 & 0.008 & 0.016 \\

$\epsilon=8/255$  
& 0.363 & 0.357 & 0.440 & 0.705 & 0.467 & 0.263 & 0.366 \\

$\epsilon=16/255$ 
& 0.630 & 0.569 & 0.742 & 0.673 & 0.654 & 0.441 & 0.591 \\

\bestcell{$\epsilon=32/255$} 
& \bestcell{\textbf{0.707}} 
& \bestcell{\textbf{0.611}} 
& \bestcell{\textbf{0.846}} 
& \bestcell{\textbf{0.628}} 
& \bestcell{\textbf{0.698}} 
& \bestcell{\textbf{0.443}} 
& \bestcell{\textbf{0.623}} \\

$\epsilon=64/255$ 
& 0.733 & 0.601 & 0.886 & 0.585 & 0.701 & 0.417 & 0.606 \\

\bottomrule
\end{tabular}
}
\label{tab:perturbation_budget}
\end{table}

\section{Additional Results}
\label{supp_sec:additional_results}

Table~\ref{tab:tlm_compare} provides the detailed numerical comparison between task-specific Training Loss Minimization (TLM) attacks and our method. The TLM variants show clear task-specific behavior: each method achieves its best performance on the task used for optimization, but the gains do not consistently transfer to other tasks. In contrast, our method obtains more balanced results across IC, OD, RC, and OL, and achieves the highest CTSR-4 score. This demonstrates that our method is more effective in generating adversarial examples that affect shared object-level semantics across tasks, rather than overfitting to a single task objective.

Table~\ref{tab:perturbation_budget} reports the detailed values corresponding to Fig.~\ref{fig:parameter}. As the perturbation budget $\epsilon$ increases, the overall attack performance improves substantially from $4/255$ to $32/255$, indicating that a larger feasible perturbation space helps the evolutionary search find stronger adversarial candidates. However, further increasing $\epsilon$ to $64/255$ does not improve the cross-task metrics, although some individual task scores become higher. Therefore, $\epsilon=32/255$ provides a better trade-off between attack strength and cross-task consistency, and is adopted as the default setting in our main experiments.

\begin{table}[!t]
\centering
\small
\renewcommand{\arraystretch}{1.08}
\caption{Summary of key symbols with explanations.}
\resizebox{0.98\columnwidth}{!}{
\begin{tabular}{ll}
\toprule
Symbol & Explanation \\
\midrule
$x$ & Clean input image. \\
$M$ & Binary mask restricting perturbations to the target object region. \\
$\delta$ & Adversarial perturbation in image space. \\
$x^{adv}$ & Final adversarial image, i.e., $x^{adv}=x+\delta$. \\
$\mathcal{T}$ & Set of visual patch indices corresponding to the target object. \\
$F_v(\cdot)$ & Visual encoder that extracts patch-level visual features. \\
$z_v^{clean}$ & Object-level visual representation of the clean image. \\
$e_s^{-}$ & Source-category text embedding used as the semantic anchor. \\
$e^{+}$ & Target-category text embedding. \\
$\mathcal{E}^{-}$ & Final evolved negative text population. \\
$K_t$ & Population size for text-side evolution. \\
$K_v$ & Population size for image-side evolution. \\
$R_t$ & Radius of the text-side $L_2$ search ball centered at $e_s^{-}$. \\
$\epsilon$ & Maximum perturbation budget in image space. \\
$\beta$ & Weight of the target-similarity penalty in text-side fitness. \\
$h$ & Weight of the negative-similarity term in the image-side loss. \\
$\phi_t$ & Mutation strength for text-side evolution. \\
$\phi_v$ & Mutation strength for image-side evolution. \\
$\mu$ & Momentum decay factor in image-side optimization. \\
$\alpha_t$ & Cosine-annealed step size at iteration $t$. \\
\bottomrule
\end{tabular}
}
\label{tab:key_symbols}
\vspace{-0.15cm}
\end{table}

\begin{figure*}[!t]
    \centering

    \subfloat[Results for umbrella$\rightarrow$handbag, cat$\rightarrow$dog, tie$\rightarrow$handbag, and train$\rightarrow$airplane.]{
        \includegraphics[
            width=0.98\textwidth,
            height=0.38\textheight,
            keepaspectratio
        ]{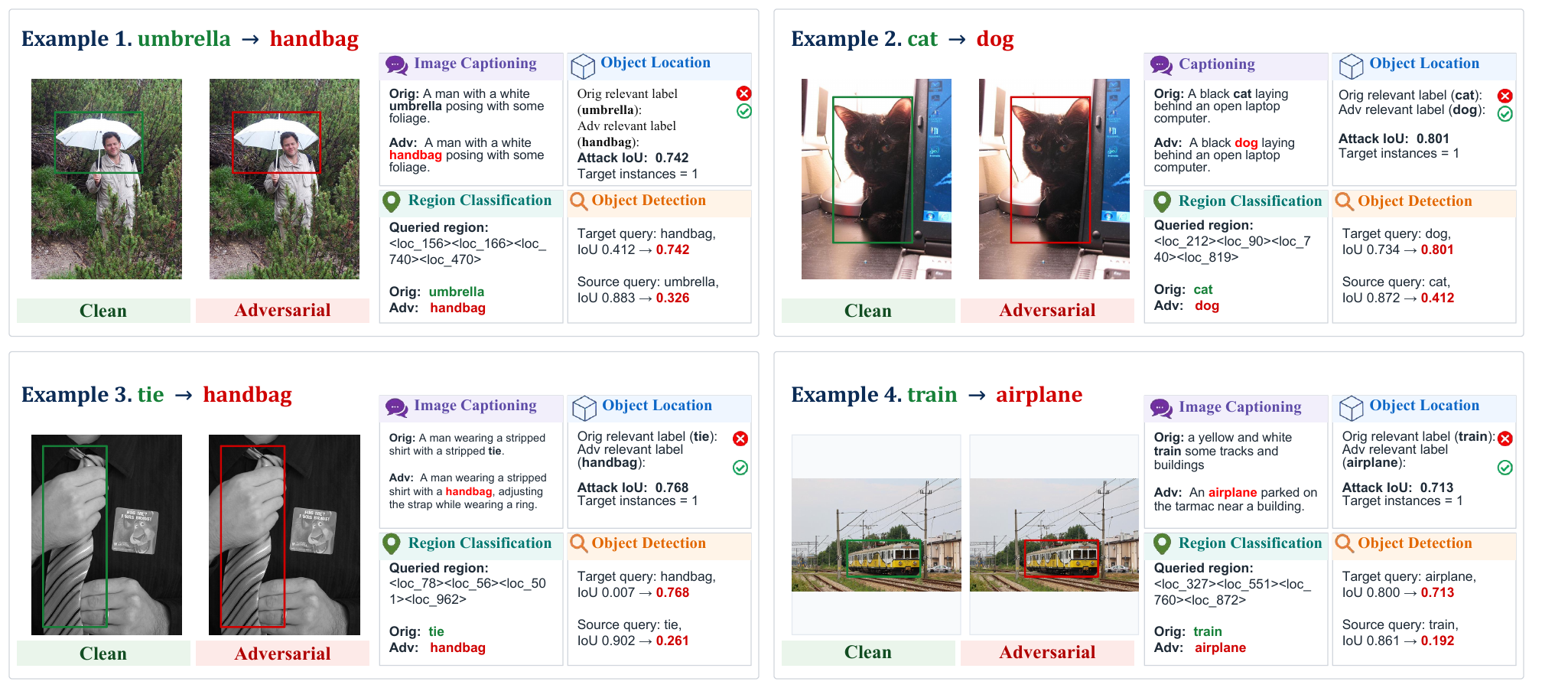}
        \label{fig:supp_example1}
    }

    \vspace{0.1cm}

    \subfloat[Results for tie$\rightarrow$handbag, cat$\rightarrow$dog, book$\rightarrow$clock, and tv$\rightarrow$laptop.]{
        \includegraphics[
            width=0.98\textwidth,
            height=0.38\textheight,
            keepaspectratio
        ]{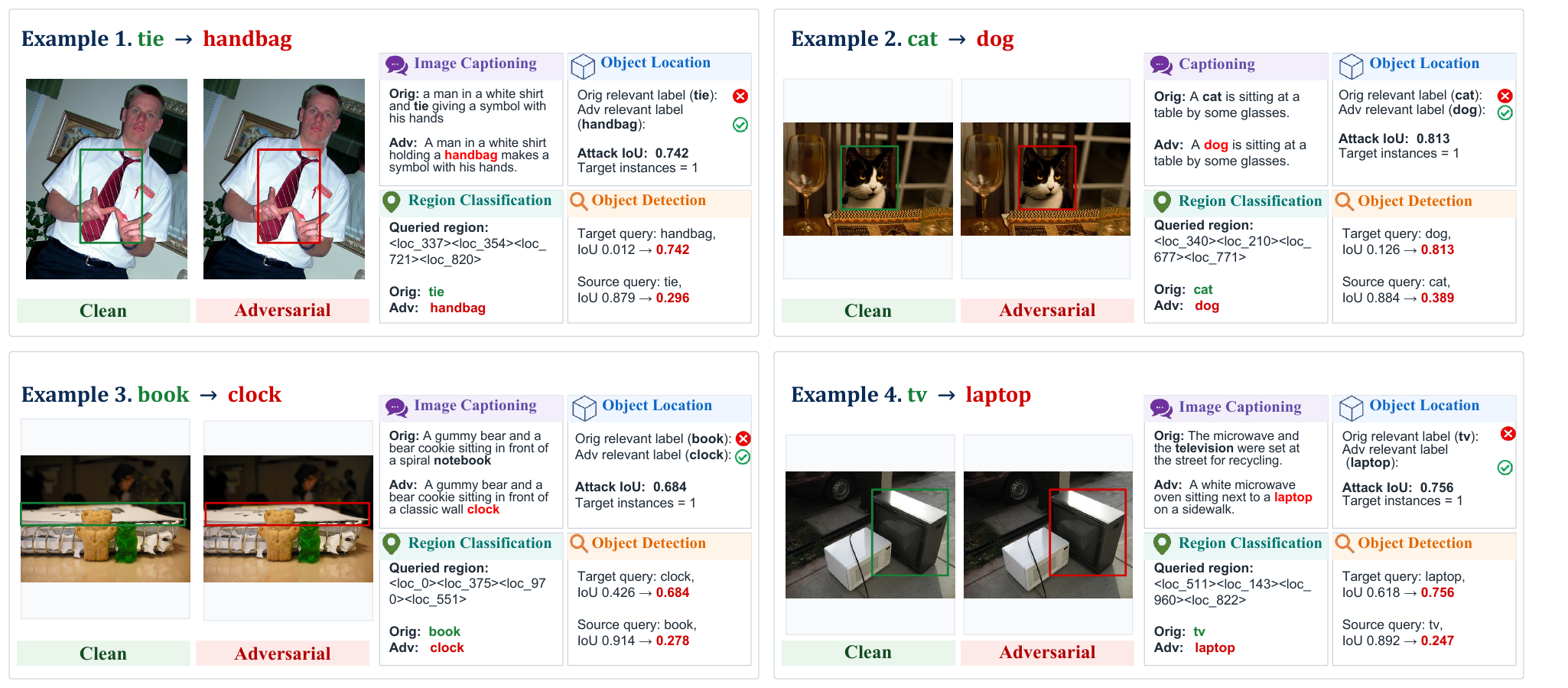}
        \label{fig:supp_example2}
    }

    \vspace{-0.15cm}
    \caption{Additional qualitative results of cross-task adversarial transfer
    under the object-change attack setting. For each source-to-target pair, we
    show the clean and adversarial images together with the corresponding
    outputs of image captioning, object localization, region categorization,
    and object detection. \colorbox{green}{Green} boxes denote source-object
    regions in clean images, while \colorbox{red}{red} boxes denote target-object
    regions associated with adversarial predictions. Green check marks and red
    cross marks indicate whether the corresponding labels are detected. The
    caption changes, category shifts, detection results, and IoU variations
    demonstrate that the intended semantic changes transfer consistently
    across multiple vision-language tasks.}
    \label{fig:supp_examples12}
    \vspace{-0.2cm}
\end{figure*}

\begin{figure*}[!t]
    \centering

    \subfloat[Results for bird$\rightarrow$cat, pizza$\rightarrow$donut, cup$\rightarrow$fork, and boat$\rightarrow$train.]{
        \includegraphics[
            width=0.98\textwidth,
            height=0.38\textheight,
            keepaspectratio
        ]{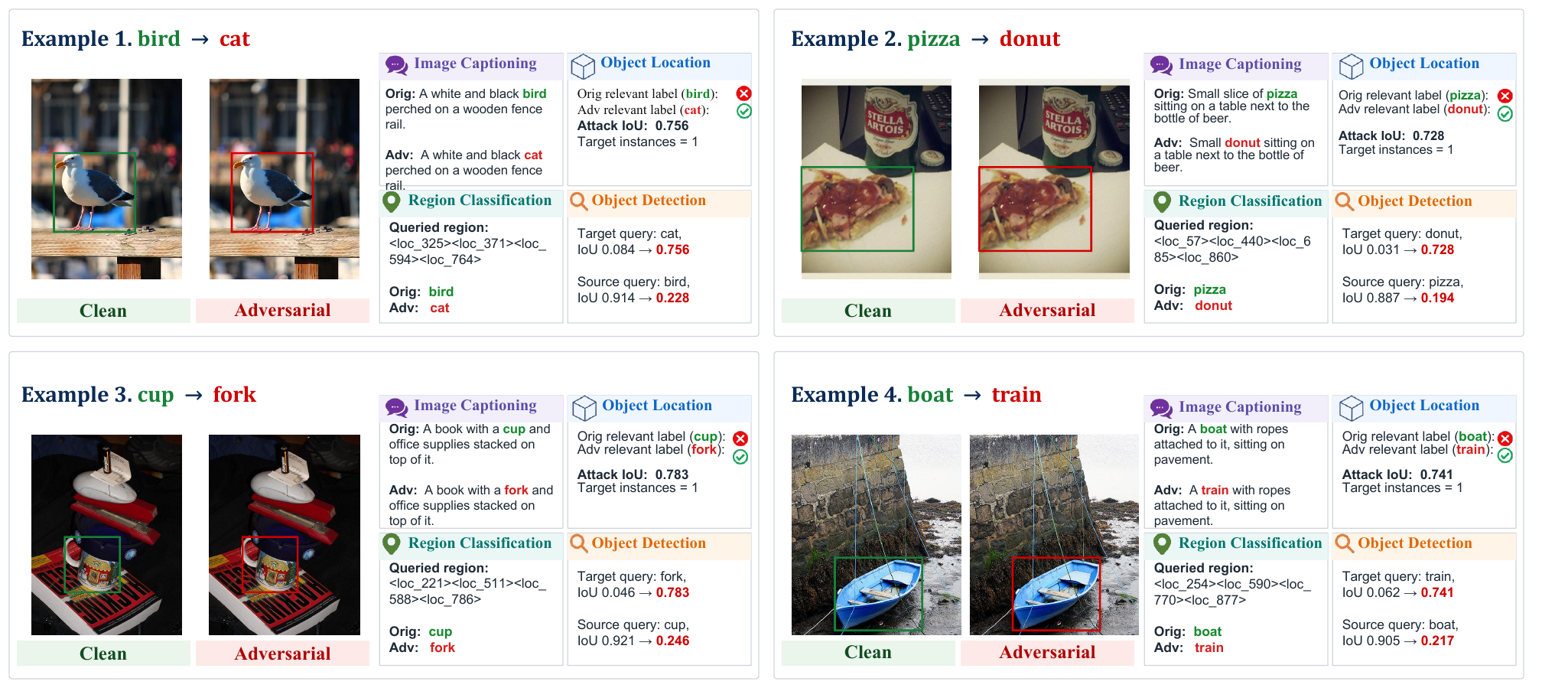}
        \label{fig:supp_example3}
    }

    \vspace{0.1cm}

    \subfloat[Results for donut$\rightarrow$cake, train$\rightarrow$airplane, skateboard$\rightarrow$surfboard, and horse$\rightarrow$sheep.]{
        \includegraphics[
            width=0.98\textwidth,
            height=0.38\textheight,
            keepaspectratio
        ]{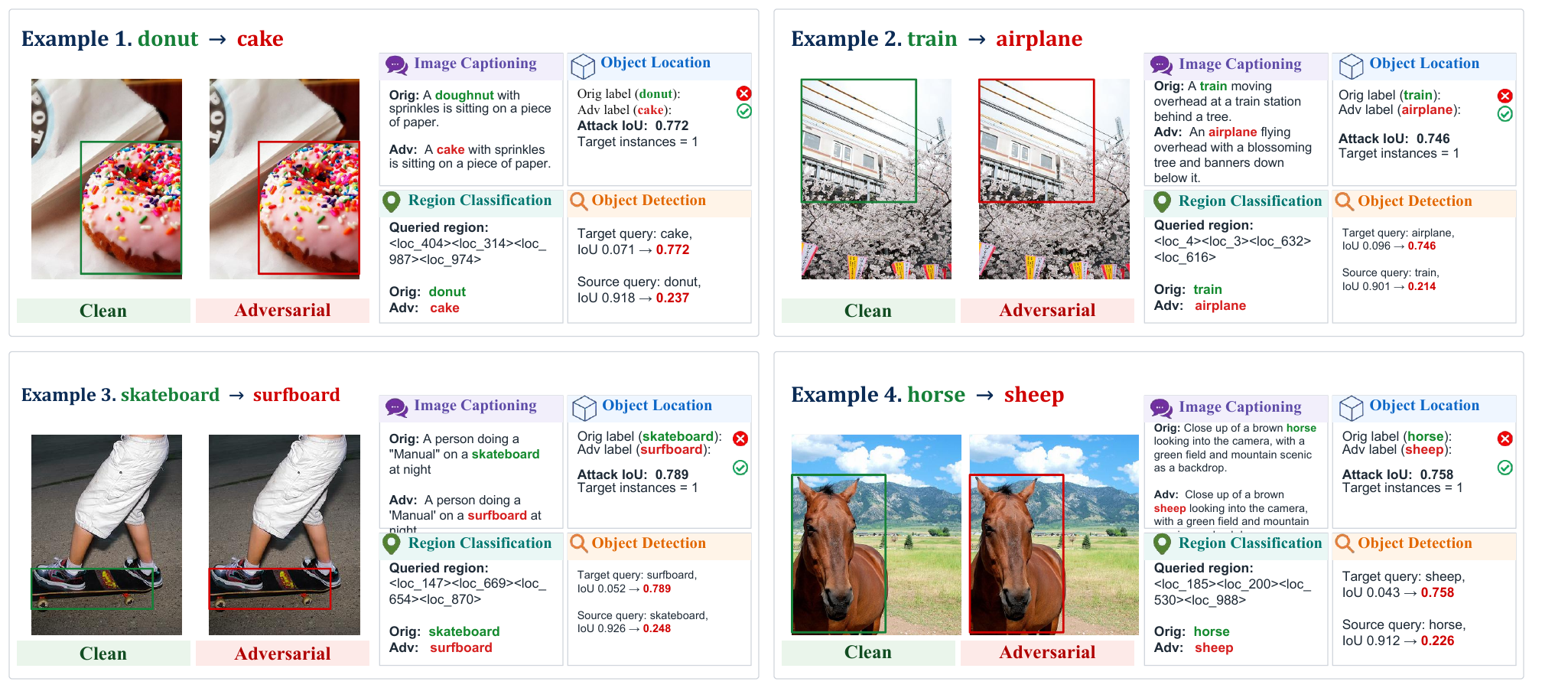}
        \label{fig:supp_example4}
    }

    \vspace{-0.15cm}
    \caption{Further qualitative results of cross-task adversarial transfer.
    The generated adversarial examples consistently suppress the source-object
    semantics and induce the target-object semantics across image captioning,
    object localization, region categorization, and object detection. The
    corresponding localization and IoU changes further confirm the
    transferability of the intended object-level semantic changes across
    different vision-language tasks.}
    \label{fig:supp_examples34}
    \vspace{-0.2cm}
\end{figure*}

\section{Summary of Symbols}
\label{supp_sec:symbols}

Table~\ref{tab:key_symbols} summarizes the key notation used in the proposed method. These symbols describe the clean and adversarial inputs, the object-region mask, the visual token set, the text-side semantic embeddings, and the main hyperparameters used in the dual-sided evolutionary optimization. This notation also clarifies the connection between the localized image perturbation and the cross-modal semantic guidance: the adversarial image is obtained by applying a bounded perturbation to the object region, while the corresponding visual tokens are used to compute object-level representations for text-side and image-side optimization.

\noindent\textbf{Limitations.}
Although \textbf{\textit{CoEvoAttack}} demonstrates strong cross-task attack performance on unified vision-language models, several limitations remain. First, our experiments are mainly conducted on three representative unified VLMs, including Florence-2, OFA, and UnifiedIO-2, and evaluated on the CrossVLAD benchmark with four object-centric tasks. While these settings provide a meaningful testbed for cross-task semantic manipulation, they do not fully cover all multimodal model families, closed-source VLMs, or more complex reasoning-oriented tasks such as visual question answering, OCR-based understanding, and embodied decision-making. Second, our method relies on source-object localization to restrict the perturbation to the object region. This design helps induce object-level semantic shifts while avoiding unnecessary whole-image perturbations, but the attack performance may be affected when the source object is small, heavily occluded, inaccurately localized, or visually entangled with surrounding objects. Third, the present work mainly focuses on digital-domain adversarial attacks under an $\ell_\infty$ perturbation constraint. Although the perturbations are bounded and localized, we have not fully evaluated human perceptual quality or physical-world robustness under printing, recapturing, lighting changes, and viewpoint variations. Finally, despite the improved search ability of the proposed co-evolutionary optimization, maintaining text-side and image-side candidate populations still introduces additional computational cost for large VLMs or resource-constrained scenarios. We view these limitations as promising directions for future research on more general, efficient, and physically robust cross-modal adversarial evaluation.

\vspace{0.1cm}
\noindent\textbf{Social Impact.}
This work studies cross-task adversarial vulnerabilities of unified vision-language models. As VLMs are increasingly used in applications such as autonomous perception, robotics, content moderation, assistive systems, medical decision support, and multimodal human--AI interaction, understanding their failure modes is important for improving reliability and trustworthiness. By showing that localized object-level perturbations can induce consistent semantic shifts across multiple vision-language tasks, our study provides useful evidence for robustness evaluation and may help motivate stronger defense mechanisms, safer benchmark protocols, and more reliable multimodal model architectures. At the same time, adversarial attack research is inherently dual-use: the proposed method could potentially be misused to manipulate the outputs of deployed VLMs, especially in safety-critical scenarios where incorrect object recognition or localization may lead to harmful decisions. For this reason, our experiments are conducted under controlled benchmark settings and are intended to reveal systematic weaknesses rather than facilitate real-world misuse. Moreover, the vulnerability exposed in this work should not be interpreted as a complete assessment of VLM safety, since trustworthy multimodal AI also involves fairness, privacy, calibration, interpretability, and responsible deployment. We hope this work contributes to a deeper security understanding of multimodal foundation models and encourages the development of more robust and responsibly deployed vision-language systems.

\section{More Visualization Results}
\label{supp_sec:visualization}

In this section, we provide additional qualitative examples to further demonstrate the cross-task transferability of the adversarial examples generated by our method. Unlike task-specific attacks that mainly affect a single output format, our method aims to induce an object-level semantic shift that can be consistently reflected across heterogeneous vision-language tasks. Therefore, for each source-to-target object pair, we present the clean and adversarial images together with the corresponding outputs of image captioning (IC), object localization (OL), region categorization (RC), and object detection (OD).

As shown in Fig.~\ref{fig:supp_example1}, the adversarial examples produce consistent semantic changes across the four evaluated tasks. For example, the source objects \emph{umbrella}, \emph{cat}, \emph{tie}, and \emph{train} are shifted toward the target categories \emph{handbag}, \emph{dog}, \emph{handbag}, and \emph{airplane}, respectively. After the attack, the generated captions explicitly contain the target categories, while the region categorization outputs change from the source labels to the corresponding target labels. Meanwhile, the localization and detection results exhibit increased target-query IoUs and decreased source-query IoUs, indicating that the target semantics become more dominant while the original source semantics are suppressed.

Fig.~\ref{fig:supp_example2} presents further examples involving different visual contexts and object categories, including \emph{tie} $\rightarrow$ \emph{handbag}, \emph{cat} $\rightarrow$ \emph{dog}, \emph{book} $\rightarrow$ \emph{clock}, and \emph{tv} $\rightarrow$ \emph{laptop}. Despite variations in object scale, background clutter, and scene composition, the generated perturbations still induce coherent target-oriented responses across IC, OL, RC, and OD. These results suggest that the attack does not merely exploit isolated task heads or specific output templates, but instead affects the shared object-level cross-modal semantics within unified vision-language models.

Figs.~\ref{fig:supp_example3} and~\ref{fig:supp_example4} provide an expanded set of source-to-target transformations under more diverse object categories and scene configurations. Across these examples, the intended target semantics are consistently reflected in the generated captions and region-level category predictions. The localization and detection outputs also respond more strongly to the target queries, while the responses associated with the original source objects are weakened or removed. The consistency of these changes across different object positions, scales, and surrounding contexts indicates that the observed cross-task transferability is not restricted to a small number of category pairs or particular visual scenes.

Overall, the supplementary visualizations in Figs.~\ref{fig:supp_example1}--\ref{fig:supp_example4} provide intuitive evidence for the cross-task behavior reported in the main experiments. The same adversarial image can simultaneously affect free-form caption generation, region-level recognition, open-vocabulary localization, and object detection. These results further demonstrate that our method generates transferable object-change adversarial examples by manipulating shared object-level semantics rather than overfitting to an individual task objective.

{
\small
\bibliographystylesupp{IEEEtran}
\bibliographysupp{supp}
}

\end{document}